\documentclass[10pt,twocolumn,letterpaper]{article}

\usepackage{cvpr}              

\usepackage[accsupp]{axessibility}
\usepackage{graphicx}
\usepackage{amsmath}
\usepackage{amsthm}
\usepackage[ruled,linesnumbered,vlined]{algorithm2e}
\usepackage{multirow}
\usepackage{amssymb}
\usepackage{amssymb}
\usepackage{booktabs}
\usepackage{pifont}
\usepackage{bm}
\usepackage[dvipsnames]{xcolor}
\usepackage{colortbl}
\definecolor{mygray}{gray}{.9}
\newcommand{\bftab}{\fontseries{b}\selectfont}
\usepackage[pagebackref,breaklinks,citecolor=blue,colorlinks]{hyperref}

\newtheorem{theorem}{Theorem}[section]

\usepackage[capitalize]{cleveref}
\crefname{section}{Sec.}{Secs.}
\Crefname{section}{Section}{Sections}
\Crefname{table}{Table}{Tables}
\crefname{table}{Tab.}{Tabs.}
\newcommand{\setParDis}{\setlength {\parskip} {0.15cm} }
\newcommand{\setParDef}{\setlength {\parskip} {0pt} }
\usepackage{multirow}
\def\cvprPaperID{8481} 
\def\confName{CVPR}
\def\confYear{2023}

\begin{document}

\title{OpenMix: Exploring Outlier Samples for Misclassification Detection}

\author{Fei Zhu$^{1,2}$, 
	Zhen Cheng$^{1,2}$, Xu-Yao Zhang$^{1,2}$\thanks{Corresponding author.}~, Cheng-Lin Liu$^{1,2}$\\
	$^1$MAIS, Institute of Automation, Chinese Academy of Sciences, Beijing 100190, China\\
	$^2$School of Artificial Intelligence, University of Chinese Academy of Sciences, Beijing, 100049, China\\
	{\tt\small \{zhufei2018, chengzhen2019\}@ia.ac.cn, \{xyz, liucl\}@nlpr.ia.ac.cn}
}
\maketitle

\begin{abstract}
   Reliable confidence estimation for deep neural classifiers is a challenging yet fundamental requirement in high-stakes applications. Unfortunately, modern deep neural networks are often overconfident for their erroneous predictions. In this work, we exploit the easily available outlier samples, i.e., unlabeled samples coming from non-target classes, for helping detect misclassification errors. Particularly, we find that the well-known Outlier Exposure, which is powerful in detecting out-of-distribution (OOD) samples from unknown classes, does not provide any gain in identifying misclassification errors. Based on these observations, we propose a novel method called OpenMix, which incorporates open-world knowledge by learning to reject uncertain pseudo-samples generated via outlier transformation. OpenMix significantly improves confidence reliability under various scenarios, establishing a strong and unified framework for detecting both misclassified samples from known classes and OOD samples from unknown classes. The code is publicly
   available at \url{https://github.com/Impression2805/OpenMix}.
\end{abstract}

\section{Introduction}
\label{sec:intro}
Human beings inevitably make mistakes, so do machine learning systems. Wrong predictions or decisions can cause various problems and harms, from financial loss to injury and death. Therefore, in risk-sensitive applications such as clinical decision making \cite{esteva2017dermatologist} and autonomous driving \cite{janai2017computer, zhang2023survey}, it is important to provide reliable confidence to avoid using wrong predictions, in particular for non-specialists who may trust the computational models without further checks. For instance, a disease diagnosis model should hand over the input to human experts when the prediction confidence is low. However, though deep neural networks (DNNs) have enabled breakthroughs in many fields, they are known to be overconfident for their erroneous predictions \cite{zhang2020towards, hendrycks2017baseline}, \emph{i.e.}, assigning high confidence for \ding{172} misclassified samples from in-distribution (ID) and \ding{173} out-of-distribution (OOD) samples from unknown classes. 

In recent years, many efforts have been made to enhance the OOD detection ability of DNNs \cite{hendrycks2019deep, hendrycks2019anomalyseg, liu2020energy, fort2021exploring, bitterwolf2022breaking, dong2022neural}, while little attention has been
paid to detecting misclassified errors from known classes. Compared with the widely studied OOD detection problem, misclassification detection (MisD) is more challenging because DNNs are typically more confident for the misclassified ID samples than that for OOD data from a different distribution \cite{granese2021doctor}. In this paper, we focus on the under-explored MisD, and propose a simple approach to help decide whether a prediction is likely to be misclassified, and therefore should be rejected.

\begin{figure}[t]
	\begin{center}
		\vskip 0.03 in
		\centerline{\includegraphics[width=1.04\columnwidth]{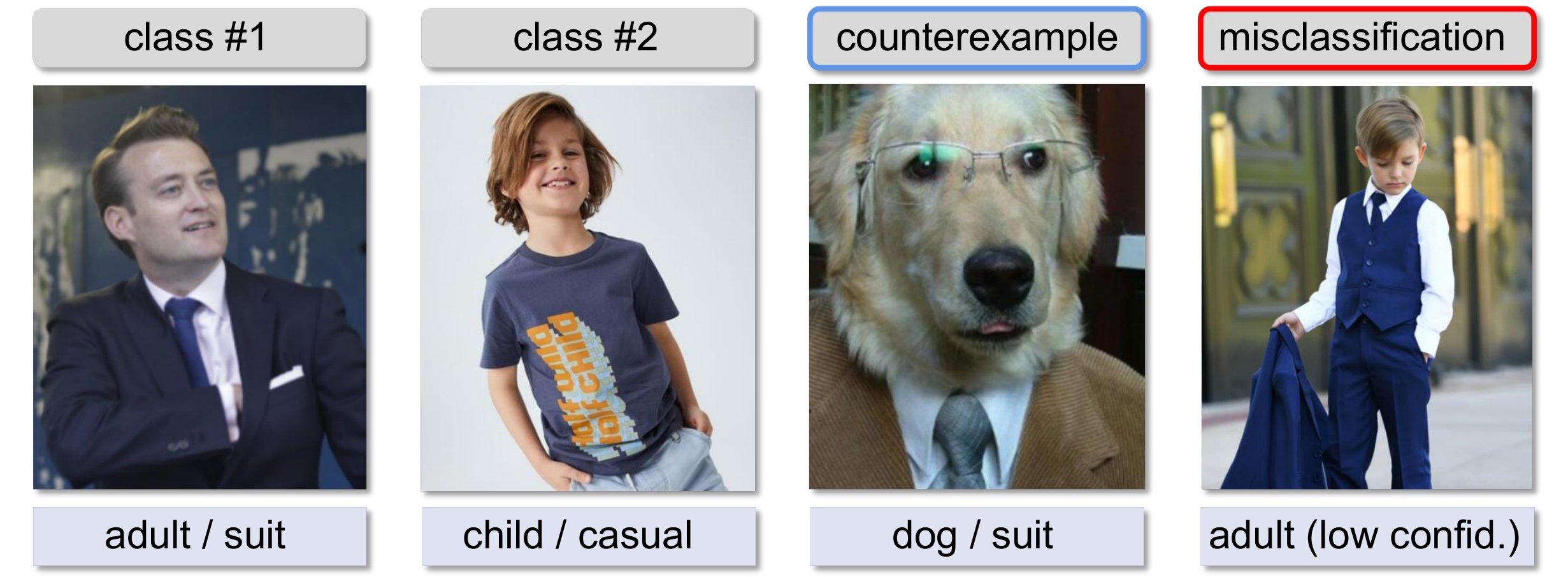}}
		\vskip -0.05 in
		\caption{Illustration of advantages of counterexample data for reliable confidence estimation. The misclassified image has the most determinative and shortcut \cite{geirhos2020shortcut} features from class \#1 (\emph{i.e.}, suit). Counterexample teaches the model the knowledge of \emph{what is not adult even if it has suit}, which could help reduce model's confidence on wrong predictions.}
		\label{figure-1}
	\end{center}
	\vskip -0.4 in
\end{figure}
Towards developing reliable models for detecting misclassification errors, we start by asking a natural question:\\ 
\centerline{\textit{Why are human beings good at confidence estimation?}}
 A crucial point is that humans learn and predict in context, where we have abundant prior knowledge about other entities in the open world. According to \emph{mental models} \cite{johnson2010mental, verschueren2005everyday, de2005working} in cognitive science, when assessing the validity or evidence of a prediction, one would retrieve counterexamples, \emph{i.e.}, which satisfy the premise but cannot lead to the conclusion. In other words, exploring counterexamples from open world plays an important role in establishing reliable confidence for the reasoning problem. Inspired by this, we attempt to equip DNNs with the above ability so that they can reduce confidence for incorrect predictions. Specifically, we propose to leverage outlier data, \emph{i.e.}, unlabeled random samples from non-target classes, as counterexamples for overconfidence mitigation. Fig.~\ref{figure-1} presents an intuitive example to illustrate the advantages of outlier samples for reducing the confidence of misclassification. 

To leverage outlier samples for MisD, we investigate the well-known
Outlier Exposure (OE) \cite{hendrycks2019deep} as it is extremely popular and can achieve state-of-the-art OOD detection performance. 
However, we find that OE is more of a hindrance than a help for identifying misclassified errors. Further comprehensive experiments show that existing popular OOD detection methods can easily ruin the MisD performance. This is undesirable as misclassified errors widely exist in practice, and a model should be able to reliably reject those samples rather than only reject OOD samples from new classes.
We observe that the primary reason for the poor MisD performance of OE and other OOD methods is that: they often compress the confidence region of ID samples in order to distinguish them from OOD samples. Therefore, it becomes difficult for the model to further distinguish
correct samples from misclassified ones. 

We propose a \emph{learning to reject} framework to leverage outlier data. \ding{172} Firstly, unlike OE and its variants which force the model to output a uniform distribution on all training classes for each outlier sample, we explicitly break the closed-world classifier by adding a separate reject class for outlier samples. \ding{173} To reduce the distribution gap between ID and open-world outlier samples, we mix them via simple linear interpolation and assign soft labels for the mixed samples. We call this method \emph{OpenMix}. Intuitively, the proposed OpenMix can introduce
the prior knowledge about \emph{what is uncertain and should be
assigned low confidence}. We provide proper justifications and show that OpenMix can significantly improve the MisD performance. We would like to highlight that our approach is simple, agnostic to the network architecture, and does not degrade accuracy when improving confidence reliability.

In summary, our primary contributions are as follows:
\begin{itemize}
	\item For the first time, we propose to explore the effectiveness of outlier samples for detecting misclassification errors. We find that OE and other OOD methods are useless or harmful for MisD. 
	\item We propose a simple yet effective method named OpenMix, which can significantly improve MisD performance with enlarged confidence separability between correct and misclassified samples. 
	\item Extensive experiments demonstrate that OpenMix significantly and consistently improves MisD. Besides, it also yields strong OOD detection performance, serving as a unified failure detection method. 
\end{itemize}

\section{Related Work}
\noindent\textbf{Misclassification detection.} Chow \cite{chow1970optimum} presented an optimal rejection rule for Bayes classifier. For DNNs, a common baseline of MisD is the maximum softmax probability (MSP) score \cite{hendrycks2017baseline}. Some works \cite{corbiere2019addressing, luo2021learning} introduce a separate confidence network to perform binary discrimination between correct and misclassified training samples. One clear drawback of those methods is that DNNs often have high training accuracy where few or even no misclassified examples exist in the training set. Moon \emph{et al.} \cite{MoonKSH20} proposed to learn an ordinal ranking relationship according to confidence for reflecting the historical correct rate during training dynamics. A recent work \cite{zhu2022rethinking} demonstrates that calibration methods \cite{guo2017calibration, MukhotiKSGTD20, MullerKH19, thulasidasan2019mixup} are harmful for MisD, and then reveals a surprising and intriguing phenomenon termed as \textbf{\emph{reliable overfitting}}: the model starts to irreversibly lose confidence reliability after training for a period, even the test accuracy continually increases. To improve MisD, a simple approach, \emph{i.e.} FMFP  \cite{zhu2022rethinking} was designed by eliminating the reliable overfitting phenomenon. A concurrent work \cite{zhu2022learning} develops \emph{classAug} for reliable confidence estimation by learning more synthetic classes.

\begin{table*}[t]
	\caption{MisD performance can not be improved with OE. AUROC and FPR95 are percentage. AURC is multiplied by $10^3$.}
	\vskip -0.07in
	\label{table-1}
	\setlength\tabcolsep{7.1pt}
	\centering
	\renewcommand{\arraystretch}{1}
	\scalebox{0.7}{
		\begin{tabular}{lrccccccccccc}
			\toprule
			\multirow{2}{*}{\textbf{Dateset}} &\multirow{2}{*}{\textbf{Method}} & \multicolumn{3}{c}{\textbf{AURC} $\downarrow$} && \multicolumn{3}{c}{\textbf{AUROC} $\uparrow$} && \multicolumn{3}{c}{\textbf{FPR95} $\downarrow$}\\
			\cmidrule(lr){3-5} \cmidrule(lr){7-9} \cmidrule(lr){11-13}
			& &ResNet110 & WRNet & DenseNet  && ResNet110 & WRNet & DenseNet  && ResNet110 & WRNet & DenseNet \\
			\midrule
			\multirow{2}{*}{CIFAR-10} 
			&MSP \cite{hendrycks2017baseline} &\bftab{9.52$\pm$0.49} &\bftab{4.76$\pm$0.62}  &\bftab{5.66$\pm$0.45} && \bftab{90.13$\pm$0.46} &\bftab{93.14$\pm$0.38} &\bftab{93.14$\pm$0.65}  &&\bftab{43.33$\pm$0.59} & \bftab{30.15$\pm$1.98}  & \bftab{38.64$\pm$4.70}\\
			&+ OE \cite{hendrycks2019deep} & 10.10$\pm$0.54 &4.83$\pm$0.13  &8.23$\pm$0.95 && 90.02$\pm$0.36 & 93.09$\pm$0.15 &91.44$\pm$0.15 &&46.89$\pm$1.78 &38.78$\pm$2.59  &45.86$\pm$2.30 \\
			\midrule
			\multirow{2}{*}{CIFAR-100} 
			&MSP \cite{hendrycks2017baseline}  &\bftab{89.05$\pm$1.39} &\bftab{46.84$\pm$0.90}  &\bftab{66.11$\pm$1.56} &&\bftab{84.91$\pm$0.13}  &\bftab{88.50$\pm$0.44} &\bftab{86.20$\pm$0.04} &&\bftab{65.65$\pm$1.72} &\bftab{56.64$\pm$1.33}  &\bftab{62.79$\pm$0.83} \\
			&+ OE \cite{hendrycks2019deep} &103.06$\pm$2.50 &58.05$\pm$1.21  &86.96$\pm$2.27 &&83.81$\pm$0.49 &86.36$\pm$0.20  &84.25$\pm$0.50  &&71.11$\pm$0.77 &62.96$\pm$0.38  &70.39$\pm$0.65  \\
			\bottomrule
		\end{tabular}
	}
	\vskip -0.1 in
\end{table*}

\setParDis
\noindent\textbf{Utilizing outlier samples.} Auxiliary outlier dataset is commonly utilized in many problem settings. For example, Lee \emph{et al.} \cite{lee2020removing} leveraged outliers to enhance adversarial robustness of DNNs. Park \emph{et al.} \cite{park2021task} used outliers to improve object localization performance. ODNL \cite{wei2021open} uses open-set outliers to prevent the model from over-fitting inherent noisy labels. In addition, outlier samples are also effective for improving few-shot learning \cite{le2021poodle} and long-tailed classification \cite{wei2022open}. In the area of confidence estimation, OE \cite{hendrycks2019deep} has been the most popular and effective way to improve OOD detection ability by using outlier samples. 

\noindent\textbf{OOD detection.} This task focuses on judging whether an input sample is from novel classes or training classes. Compared with MisD, OOD detection has been studied extensively in recent years and various methods have been developed, including training-time \cite{hendrycks2019deep, wei2022logitnorm, techapanurak2020hyperparameter, cao2022deep, wang2022partial} and post-hoc strategies \cite{LiangLS18, lee2018simple, liu2020energy, hendrycks2019anomalyseg, dong2022neural}. Cheng \emph{et al.} \cite{cheng2023average} proposed a AoP (Average of Pruning) framework to improve the performance and stability of OOD detection, which also offers notable gain for MisD. Most existing OOD detection works do not involve detecting misclassified errors. We would like to highlight that both OOD and misclassified samples are failure sources and should be rejected together.
\setParDef

\section{Problem Setting and Motivation}
\subsection{Preliminaries: MisD and OE}
\noindent\textbf{Basic notations.} 
Let $\mathcal{X} \in \mathbb{R}^d$ denote the input space and $\mathcal{Y} = \{1,2,...,k\}$ represents the label space. 
Given a sample ($\bm{x}, ~y$) drawn from an unknown distribution $\mathcal{P}$ on $\mathcal{X} \times \mathcal{Y}$, a neural
network classifier $f(\cdot): \mathbb{R}^d \rightarrow \Delta_k$ produces a probability distribution for $\bm{x}$ on $k$ classes, where $\Delta_k$ denotes the $k-1$ dimensional simplex. Specifically, $f_i(\bm{x})$ denotes the $i$-th element of
the softmax output vector produced by $f$. Then $\hat{y} =: \arg\max_{y \in \mathcal{Y}} f_y(\bm{x})$ can be returned as the predicted class and the associated probability $\hat{p} =: \max_{y \in \mathcal{Y}} f_y(\bm{x})$ can be viewed as the predicted confidence. 
Denote by $\mathcal{D}_\text{in}$ the distribution over $\mathcal{X}$ of ID data. Besides, we can also have access to some unlabeled outlier samples (\emph{i.e.}, $\mathcal{D}_\text{out}$) coming from outside target classes.
At inference time, most of the inputs are from known classes, and they can be correctly classified or misclassified. We use $\mathcal{D}^{\text{test},\checkmark}_\text{in}$ and $\mathcal{D}^{\text{test},\bm{\times}}_\text{in}$ to represent the distribution of correct and misclassified ID samples, respectively.

\setParDis
\noindent\textbf{Misclassification detection.} MisD, also known as failure prediction \cite{corbiere2019addressing, zhu2022rethinking}, is a critical safeguard for safely deploying machine learning
models in real-world applications. It focuses on detecting and filtering wrong predictions ($\mathcal{D}^{\text{test},\bm{\times}}_\text{in}$) from correct predictions ($\mathcal{D}^{\text{test},\checkmark}_\text{in}$) based on their confidence ranking. Formally, denote $\kappa$ a confidence-rate function (\emph{e.g.}, the MSP or negative entropy) that assesses the degree of confidence of the predictions, with a predefined threshold $\delta \in \mathbb{R}^{+}$, the misclassified samples can be detected based on a decision function $g$ such that for a given input $\bm{x}_i \in \mathcal{X}$:
\begin{equation}
\label{eq1}
g(\bm{x}_i) = \left\{ 
\begin{aligned}
 &\text{correct}~~~~\text{if} ~~\kappa(\bm{x}_i) \ge \delta, \\
 &\text{misclassified}~~~~\text{otherwise}.
\end{aligned}
\right.
\end{equation} 

\noindent\textbf{Outlier Exposure.} OE \cite{hendrycks2019deep} leverages auxiliary outliers to help the model detect
OOD inputs by assigning low confidence for samples in $\mathcal{D}_{\text{out}}$. Specifically, given a model $f$ and the original learning objective
$\ell_\text{CE}$ (\emph{i.e.}, cross-entropy loss), OE minimizes the following objective:
\begin{equation}\label{eq2}
\mathbb{E}_{\mathcal{D}^{\text{train}}_\text{in}} [\ell_\text{CE}(f(\bm{x}),y)] 
+ \lambda~ \mathbb{E}_{\mathcal{D}_\text{out}} [\ell_\text{OE}(f(\widetilde{\bm{x}}))],
\end{equation}
where $\lambda > 0$ is a penalty hyper-parameter, and $\ell_\text{OE}$ is defined by
Kullback-Leibler (KL) divergence to the uniform distribution: $\ell_\text{OE}(f(\bm{x})) = \text{KL}(\mathcal{U}(y) \Arrowvert
 f(\bm{x}))$, in which $\mathcal{U}(\cdot)$ denotes the uniform distribution.
Basically, OE uses the available OOD data $\mathcal{D}_\text{out}$ to represent the real OOD data that would be encountered in open environments. Although the limited samples in $\mathcal{D}_\text{out}$ can not fully reveal the real-world OOD data, OE surprisingly yields strong performance in OOD detection. The strong effectiveness of outliers for improving OOD detection has been verified by many recent works \cite{liu2020energy, liznerski2022exposing}. This leads us to ask: \emph{Can we use outlier data to help detect misclassification errors?}\setParDef

\begin{figure}[h]
	\begin{center}
		\centerline{\includegraphics[width=\columnwidth]{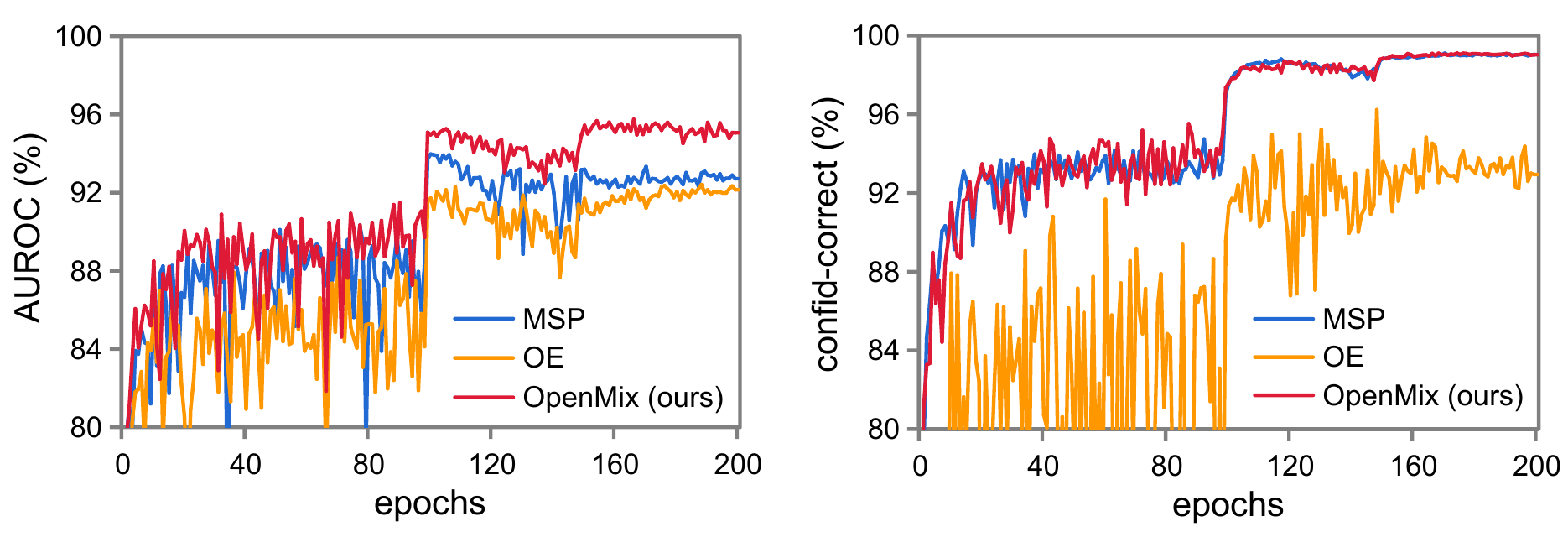}}
		\vskip -0.1 in
		\caption{The AUROC and averaged confidence of correct samples under different training epochs. OE results in (left) worse AUROC with (right) under-confident correctly classified samples.}
		\label{figure-2}
	\end{center}
	\vskip -0.35 in
\end{figure}
\subsection{Motivation: understanding the effect of OE}
\noindent We start with the empirical experiments of OE, analyzing the role of outlier data for MisD. Throughout this subsection, we perform experiments on CIFAR \cite{krizhevsky2009learning} using
standard cross-entropy loss and OE based training, respectively. We use 300K {\ttfamily RandImages} as the OOD auxiliary
dataset following \cite{hendrycks2019deep, wei2021open, wei2022open}. Specifically, all images that
belong to CIFAR classes are removed in {\ttfamily RandImages} so that $\mathcal{D}_\text{in}$ and $\mathcal{D}_\text{out}$ are disjoint. 
Evaluation metrics include AURC $\downarrow$ \cite{GeifmanE17}, FPR95 $\downarrow$ and AUROC $\uparrow$ \cite{davis2006relationship}.

\setParDis
\noindent\textbf{OE has negative impact on MisD.} Table~\ref{table-1} presents the results of training
without/with the auxiliary outlier dataset. We can observe that OE consistently deteriorates the MisD performance under various metrics. For example, when training with OE on CIFAR-10/WRNet, the FPR95$\downarrow$ increases $8.63$ percentages compared with baseline, \emph{i.e.}, MSP. In Fig.~\ref{figure-2} (left), we can observe that the AUROC of OE is consistently lower than that of baseline method during training of WRNet on CIFAR-10. Intuitively, to distinguish correct predictions from errors, the model should assign high confidence for correct samples, and low confidence for errors. However, in Fig.~\ref{figure-2} (right), we find that OE can significantly deteriorate the confidence of correct samples, which makes it difficult to separate correct and wrong predictions.

\noindent\textbf{Understanding from feature space uniformity.} Overconfidence for misclassified prediction implies that the sample is projected into the density region of a wrong class \cite{zhu2022rethinking}. Intuitively, excessive feature compression would lead to over-tight class distribution, increasing the overlap between correct and misclassified samples.
To better understand the negative effect of OE for MisD, we study its impact on the learned deep feature space. 
Let $z(\cdot)$ represent the feature extractor, we then define and compute the inter-class distances $\pi_{inter} =  \frac{1}{Z_{inter}} \sum_{y_{l},y_{k},l \neq k}d(\bm{\mu}(Z_{y_{l}}), \bm{\mu}(Z_{y_k}))$, and average intra-class distances $\pi_{intra} =  \frac{1}{Z_{intra}} \sum_{y_{l} \in y} \sum_{\bm{z}_{i},\bm{z}_{j} \in Z_{y_{l}},i \neq j}d(\bm{z}_{i}, \bm{z}_{j})$, in which $d(\cdot;\cdot)$ is the distance function. $Z_{y_{l}}=\{\bm{z}_{i}:=z(\bm{x}_{i})|y_{i}=y_{l}\}$ denotes the set of deep feature vectors of samples in class $y_{l}$. $\bm{\mu}(Z_{y_{l}})$ is the class mean. $Z_{intra}$ and $Z_{inter}$ are two normalization constants. Finally, the feature space uniformity (FSU) is defined as $\pi_{fsu} = \pi_{intra}/\pi_{inter}$  \cite{roth2020revisiting}. Intuitively, large FSU increases the instances in low density regions and encourages the learned features to distribute uniformly (maximal-info preserving) in feature space. \setParDef

When facing OOD samples from new classes, small FSU (larger inter-class distance and small intra-class distance) could result in less overlap between ID and OOD samples. However, compared to OOD data, misclassified samples are ID and distributed much closer to correct samples of each class. As shown in Fig.~\ref{figure-3}, the FSU is reduced with OE. By forcing the outliers to be uniformly distributed over original classes, OE introduces similar effect as label-smoothing \cite{MullerKH19}, which leads to over-compressed distributions, losing the important information about the hardness of samples. 
Consequently, ID samples of each class would be distributed within a compact and over-tight region, making it harder to separate misclassified samples from correct ones. 
Supp.M provides a unified view on the connection between FSU and OOD detection, MisD performance.
\begin{figure}[h]
	\begin{center}
		\vskip -0.05 in
		\centerline{\includegraphics[width=1.07\columnwidth]{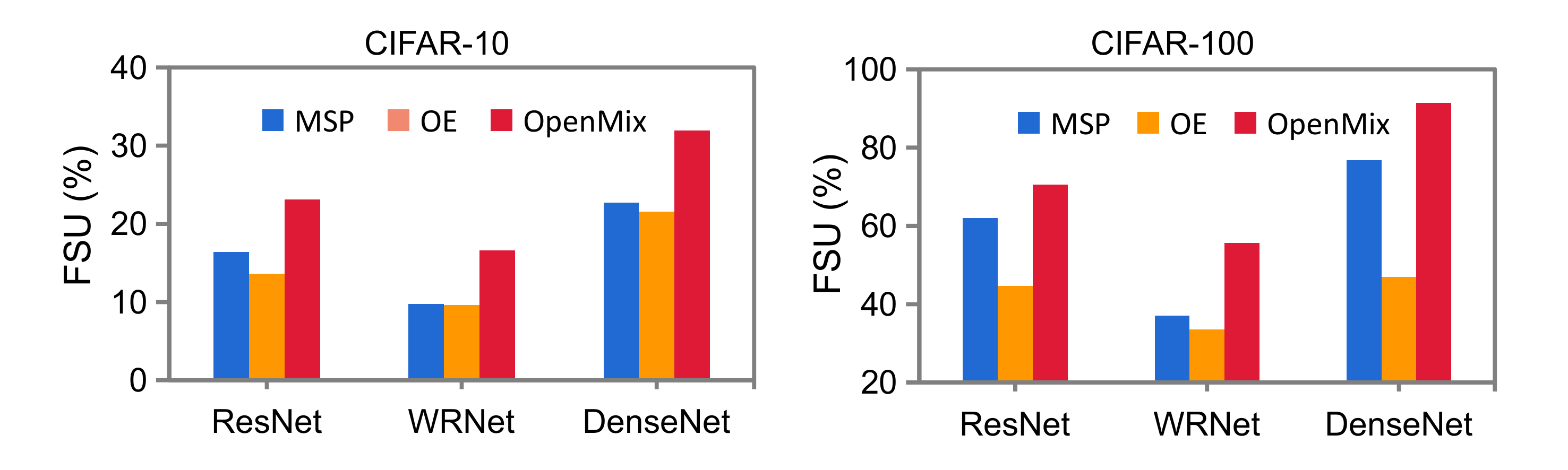}}
		\vskip -0.1 in
		\caption{The impact of OE on the deep feature space. With OE, the feature space uniformity (FSU) is reduced, which indicates excessive feature compression and has negative influence for MisD. Our OpenMix leads to less compact feature distributions.}
		\label{figure-3}
	\end{center}
	\vskip -0.35 in
\end{figure}

\setParDis\noindent\textbf{How to use outliers for MisD?} Based on the above observations and analysis, we argue that the original OE \cite{hendrycks2019deep} should be modified from two aspects for MisD:\setParDef
\begin{itemize}
	\item \emph{On learning objective.} Simply forcing the model to yield uniform distribution for outliers with $\ell_\text{OE}$ would lead to reduced feature space uniformity and worse MisD performance. We suggest that the original $\ell_\text{OE}$ loss should be discarded, and a new learning objective to use outliers should be designed.
	\item \emph{On outlier data.} Outliers from unknown classes are OOD samples and have a large distribution gap with ID misclassified samples, which could weaken the effect for MisD. To overcome this issue, we suggest transforming available outlier data into new outliers that are distributed closer to ID misclassified samples.
\end{itemize}
Motivated by the above observations and analysis, we propose to modify OE from the perspective of learning objective and outlier data, respectively.

\section{Proposed Method: OpenMix}
\begin{figure}[t]
	\begin{center}
		\centerline{\includegraphics[width=\columnwidth]{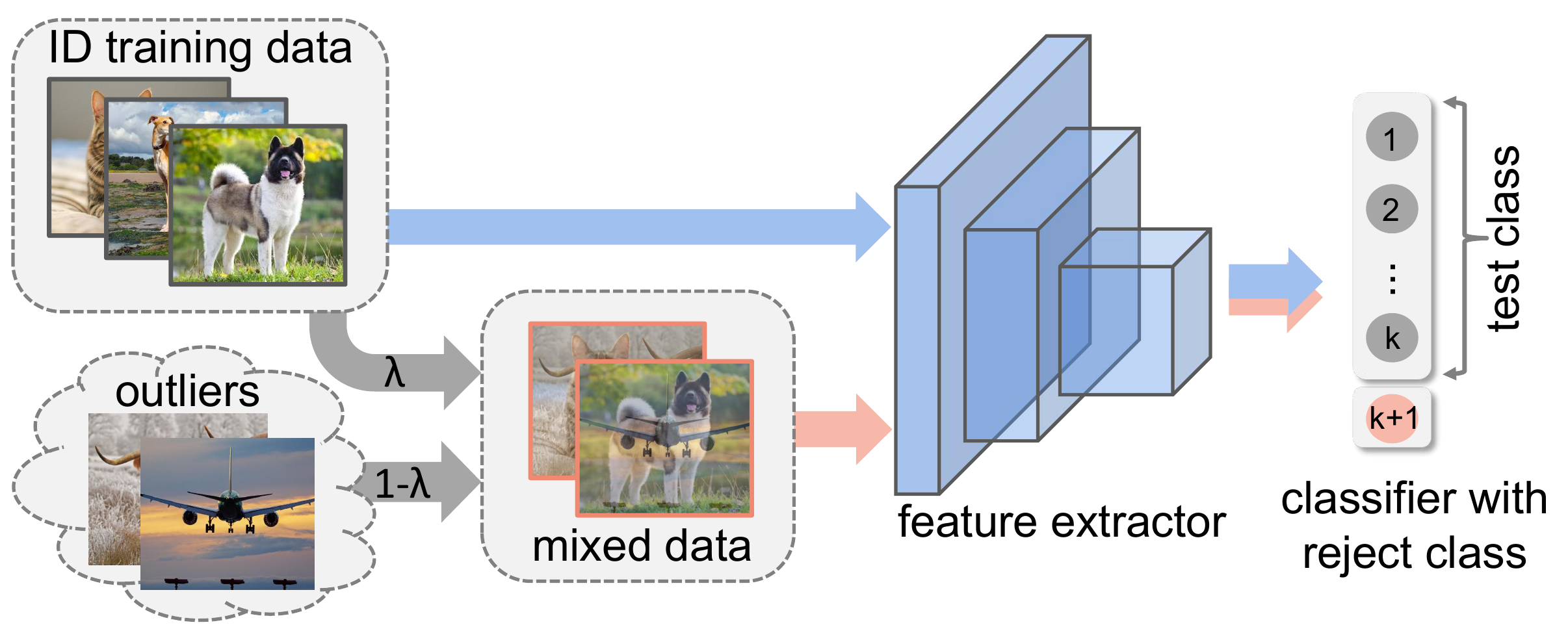}}
		\vskip -0.1 in
		\caption{The pipeline of OpenMix.}
		\label{figure-4}
	\end{center}
	\vskip -0.35 in
\end{figure}
\noindent\textbf{Learning with reject class.} Different from OE that forces the model to output uniform distribution, we propose to predict the outliers as an additional reject class. Specifically, for a $k$-class classification problem, we extend the label space by explicitly adding a separate class for outlier samples. Formally, denote $\mathbb I^{y_i}:=(0,...,1,..,0)^\top \in \{0,1\}^{k+1}$ is a one-hot vector and only the $y_i$-th entry is 1. For the outlier dataset, we map the samples to the $(k+1)$-class. The learning objective is:
\begin{equation}\label{eq3}
\mathcal{L}_{\text{total}} = \mathbb{E}_{\mathcal{D}^{\text{train}}_\text{in}} [\ell(f(\bm{x}),y)] + \gamma \mathbb{E}_{\mathcal{D}_\text{out}} [\ell(f(\widetilde{\bm{x}}),\widetilde{y})],
\end{equation}
where $\widetilde{y}=k+1$ and $\gamma$ denotes a hyper-parameter.
With reject class, the negative effect of outliers for MisD could be alleviated. However, there is little performance gain compared with baseline method, as will be shown in Sec. 5.2. Intuitively, the best auxiliary samples are the misclassified examples. However, OOD outliers can not represent misclassified ID samples well due to the distribution gap. 

\setParDis\noindent\textbf{Outlier transformation via Mixup.}
The distribution gap existing between misclassified ID samples and the OOD outliers significantly limits the effectiveness of learning with reject class. To address this issue, we propose a simple yet powerful strategy to shrink the gap by transforming the original outliers to be near the ID distribution. Specifically, inspired by the well-known Mixup technique \cite{zhang2018mixup}, we perform simple linear interpolation between ID training samples and OOD outliers. Formally, Given a pair of examples $(\bm{x}, y)$ and $(\widetilde{\bm{x}}, \widetilde{y})$ respectively sampled from the ID training set and outlier data, we apply linear interpolation to produce transformed outlier $(\breve{\bm{x}}, \breve{y})$ as follows:
\begin{equation}
\label{eq4}
\begin{aligned}
\breve{\bm{x}} = \lambda \bm{x} + (1-\lambda) \widetilde{\bm{x}}, ~~~
\mathbb I^{\breve{y}} = \lambda \mathbb I^{y} + (1-\lambda) \mathbb I^{\widetilde{y}}.
\end{aligned}
\end{equation}
The $\lambda \in [0,1]$ is a parameter sampled as $\lambda\sim\text{Beta}(\alpha, \alpha)$ for $\alpha \in (0, \infty)$. $y \in \{1,...,k\}$, $\widetilde{y}=k+1$ and $\mathbb I^{\breve{y}}$ denotes the one-hot label. Compared with Mixup \cite{zhang2018mixup}, our method involves outliers and makes sure that one of the interpolated labels always belongs to the added class, \emph{i.e.}, the ($k+1$)-th class. As shown in Sec. 5.2, other interpolation strategies like CutMix \cite{yun2019cutmix} and Manifold Mixup \cite{verma2019manifold} can also be used.

\noindent\textbf{Final learning objective.}
Combining reject class with outlier transformation, the final training objective of our OpenMix is as follows:
\begin{equation}
\label{eq5}
\begin{aligned}
\mathcal{L}_{\text{total}} = \mathbb{E}_{\mathcal{D}^{\text{train}}_\text{in}} [\ell(f(\bm{x}),y)] + \gamma\mathbb{E}_{\mathcal{D}^{\text{mix}}_\text{out}} [\ell(f(\breve{\bm{x}}),\breve{y})] \\
=\mathbb{E}_{\mathcal{D}^{\text{train}}_\text{in}} [-\mathbb I^{y}~\text{log}~f(\bm{x})] + \gamma\mathbb{E}_{\mathcal{D}^{\text{mix}}_\text{out}} [-\mathbb I^{\breve{y}}~\text{log}~f(\breve{\bm{x}})].
\end{aligned}
\end{equation}
In practice, we do not produce all mixed samples beforehand, but apply the outlier transformation in each mini-batch during training like Mixup. The details of OpenMix are provided in Algorithm~1, and Fig.~\ref{figure-4} illustrates the overall framework of OpenMix.

\noindent\textbf{Inference.} Our method focuses on detecting misclassified samples from known classes. Therefore, only the original $k$ classes are evaluated in test phase. Specifically, the predicted label of an input $\hat{y} =: \arg\max_{y \in \mathcal{Y}} f_y(\bm{x})$ and the corresponding confidence is the common MSP score, \emph{i.e.}, $\hat{p} =: \max_{y \in \mathcal{Y}} f_y(\bm{x})$, in which $\mathcal{Y} = \{1,2,...,k\}$.
\begin{algorithm}[t]
	\caption{OpenMix for MisD}
	\KwIn{Training dataset $\mathcal{D}^{\text{train}}_\text{in}$. Outlier dataset $\mathcal{D}_\text{out}$.}
	
	\For{each iteration}{ 
		Sample a mini-batch of ID training data $\{(\bm{x}_i, y_i)\}_{i=1}^{n}$ from $\mathcal{D}^{\text{train}}_\text{in}$;\\
		Sample a mini-batch of OOD outlier data $\{\widetilde{\bm{x}}_i\}_{i=1}^{n}$ from $\mathcal{D}_\text{out}$;\\
		Generate transformed outlier data $\{(\breve{\bm{x}}_i, \breve{y}_i)\}_{i=1 }^{n}$ based on Eq.~\ref{eq4};\\
		Perform common gradient descent on $f$ with $\mathcal{L}_{\text{total}}$ based on Eq.~\ref{eq5};
	}
\end{algorithm}

\noindent\textbf{Why OpenMix is beneficial for MisD?} Here we provide an interpretation: \textbf{\emph{OpenMix increases the exposure of low density regions}}. In standard training, it is difficult for reliable confidence learning because the low density regions (uncertain regions) are often under-explored, where few data points are mapped to those regions. This is expected as cross-entropy loss forces all samples to be correctly classified by matching their probability distributions with one-hot labels. As a result, the low density regions with rich uncertainty are largely ignored, leading to overconfidence for incorrect predictions. With OpenMix, the samples synthesized via outlier transformation, \emph{i.e.}, mixup of the outlier and ID regions, could reflect the property of low density regions, and soft labels teach the model to be uncertain for those samples. The results in Fig.~\ref{figure-3} confirm that OpenMix can effectively enlarge the FSU with increased exposure of low density regions.
Besides, by keeping one of the classes in soft labels always belonging to the $(k+1)$ class, OpenMix can keep the confidence of correct samples over original $k$ classes, as shown in Fig.~\ref{figure-2} (right). Supp.M provides a theoretical justification showing that OpenMix increases the exposure of low density regions.

\setParDef\section{Experiments}
\noindent\textbf{Datasets and networks.} We conduct a thorough empirical evaluation on benchmark datasets CIFAR-10 and CIFAR-100 \cite{krizhevsky2009learning}. For network architectures, we consider a wide range of DNNs such as ResNet110 \cite{HeZRS16}, WideResNet \cite{zagoruyko2016wide} and DenseNet \cite{HuangLMW17}. We use 300K {\ttfamily RandImages} \cite{hendrycks2019deep} as the auxiliary outlier data and more discussions on the different choices of outlier datasets are presented in Sec. 5.2. Besides, the results of large-scale experiments on ImageNet \cite{deng2009imagenet} with ResNet-50 \cite{he2016deep} are also reported.

\begin{table*}[t]
	\caption{Mean and standard deviations of MisD performance on CIFAR benchmarks. The experimental results are reported over three trials. The best mean results are bolded. AUROC, FPR95 and Accuracy are percentages. AURC is multiplied by $10^3$.}
	\vskip -0.07in
	\label{table-2}
	\setlength\tabcolsep{6pt}
	\centering
	\renewcommand{\arraystretch}{1}
	\scalebox{0.67}{
		\begin{tabular}{llccccccccc}
			\toprule
			\multirow{2}{*}{\textbf{Network}} &\multirow{2}{*}{\textbf{Method}} & \multicolumn{4}{c}{\textbf{CIFAR-10}} && \multicolumn{4}{c}{\textbf{CIFAR-100}}\\
			\cmidrule(lr){3-6} \cmidrule(lr){8-11}
			& &\textbf{AURC} $\downarrow$ & \textbf{AUROC} $\uparrow$ & \textbf{FPR95} $\downarrow$  &\textbf{ACC} $\uparrow$  &&\textbf{AURC} $\downarrow$ & \textbf{AUROC} $\uparrow$ & \textbf{FPR95} $\downarrow$  &\textbf{ACC} $\uparrow$ \\
			\midrule
			\multirow{9}{*}{ResNet110} 
			&MSP~{\scriptsize\textcolor{darkgray}{[ICLR17]}} \cite{hendrycks2017baseline}  &9.52$\pm$0.49 &90.13$\pm$0.46 & 43.33$\pm$0.59 & 94.30$\pm$0.06 &&89.05$\pm$1.39 &84.91$\pm$0.13 &65.65$\pm$1.72 &73.30$\pm$0.25\\
			&Doctor~{\scriptsize\textcolor{darkgray}{[NeurIPS21]}} \cite{granese2021doctor}  &9.51$\pm$0.49 &90.15$\pm0.44$  &42.95$\pm0.78$ & 94.30$\pm$0.06 &&89.84$\pm1.12$  &84.94$\pm0.09$ &64.75$\pm1.37$ &73.30$\pm$0.25 \\
			&ODIN~{\scriptsize\textcolor{darkgray}{[ICLR18]}} \cite{LiangLS18}  &20.82$\pm$1.09 & 79.45$\pm$0.75 & 59.32$\pm$1.08 & 94.30$\pm$0.06 &&167.53$\pm$9.93 & 68.95$\pm$1.95 & 79.64$\pm$1.43 & 73.30$\pm$0.25 \\
			&Energy~{\scriptsize\textcolor{darkgray}{[NeurIPS20]}} \cite{liu2020energy} &15.13$\pm$0.85 & 84.72$\pm$0.80 & 53.89$\pm$0.65  & 94.30$\pm$0.06 &&128.66$\pm$5.05 & 76.80$\pm$1.07 & 73.54$\pm$0.73 & 73.30$\pm$0.25\\
			&MaxLogit~{\scriptsize\textcolor{darkgray}{[ICML22]}} \cite{hendrycks2019anomalyseg}  &14.93$\pm$0.87 & 85.00$\pm$0.80 & 53.01$\pm$1.13 & 94.30$\pm$0.06 &&125.38$\pm$4.54 & 77.73$\pm$0.96 & 70.61$\pm$0.70 & 73.30$\pm$0.25\\
			&LogitNorm~{\scriptsize\textcolor{darkgray}{[ICML22]}} \cite{wei2022logitnorm}  & 12.57$\pm$1.32 & 88.82$\pm$0.84 & 56.27$\pm$2.61 & 92.64$\pm$0.23 && 118.00$\pm$3.17 & 79.56$\pm$0.16 & 73.09$\pm$0.18 & 71.68$\pm$0.34\\
			&Mixup~{\scriptsize\textcolor{darkgray}{[NeurIPS18]}} \cite{zhang2018mixup} & 16.27$\pm$1.33 & 86.21$\pm$0.83 &40.71$\pm$0.88 &94.69$\pm$0.31 && 87.39$\pm$1.83 & 84.60$\pm$0.88 &64.95$\pm$3.28 &75.08$\pm$0.30\\
			&RegMixup~{\scriptsize\textcolor{darkgray}{[NeurIPS22]}} \cite{pinto2022regmixup}  &7.88$\pm$0.64 &89.40$\pm$0.64  &50.91$\pm$1.47 &\bftab{95.10$\pm$0.23} &&75.76$\pm$2.00 &84.80$\pm$0.48 &64.75$\pm$1.16 &\bftab{76.15$\pm$0.14}\\
			\cmidrule(lr){2-11}
			&OpenMix (ours) &\bftab{6.31$\pm$0.32} &\bftab{92.09$\pm$0.36}  &\bftab{39.63$\pm$2.36} &94.98$\pm$0.20 &&\bftab{73.84$\pm$1.31} &\bftab{85.83$\pm$0.22} &\bftab{64.22$\pm$1.35} &75.77$\pm$0.35\\
			\midrule
			\midrule
			\multirow{9}{*}{WRNet} 
			&MSP~{\scriptsize\textcolor{darkgray}{[ICLR17]}} \cite{hendrycks2017baseline}  &4.76$\pm$0.62 &93.14$\pm$0.38 &30.15$\pm$1.98 &95.91$\pm$0.07 &&46.84$\pm$0.90  &88.50$\pm$0.44 & 56.64$\pm$1.33 & 80.76$\pm$0.18\\
			&Doctor~{\scriptsize\textcolor{darkgray}{[NeurIPS21]}} \cite{granese2021doctor}  &4.75$\pm$0.61 &93.13$\pm0.38$  &30.46$\pm1.90$ &95.91$\pm$0.07 &&47.34$\pm1.31$  &88.41$\pm0.23$ &57.64$\pm0.64$ & 80.76$\pm$0.18\\
			&ODIN~{\scriptsize\textcolor{darkgray}{[ICLR18]}} \cite{LiangLS18} &20.37$\pm$3.36  & 74.70$\pm$2.67  & 62.04$\pm$2.86 & 95.91$\pm$0.07 &&72.58$\pm$0.69 & 81.02$\pm$0.37 & 65.22$\pm$0.53 & 80.76$\pm$0.18\\
			&Energy~{\scriptsize\textcolor{darkgray}{[NeurIPS20]}} \cite{liu2020energy}  &6.91$\pm$0.66 & 90.47$\pm$0.51 & 39.13$\pm$2.07 & 95.91$\pm$0.07&&57.30$\pm$1.24  & 85.05$\pm$0.34 & 64.15$\pm$0.26 & 80.76$\pm$0.18\\
			&MaxLogit~{\scriptsize\textcolor{darkgray}{[ICML22]}} \cite{hendrycks2019anomalyseg} &6.85$\pm$0.66 & 90.60$\pm$0.52 & 37.01$\pm$2.38 & 95.91$\pm$0.07 &&56.07$\pm$1.24 & 85.62$\pm$0.32 & 61.57$\pm$0.56  & 80.76$\pm$0.18\\
			&LogitNorm~{\scriptsize\textcolor{darkgray}{[ICML22]}} \cite{wei2022logitnorm}  & 5.81$\pm$0.45 & 91.06$\pm$0.26 & 46.06$\pm$2.24 & 95.50$\pm$0.33 && 72.05$\pm$1.32 & 82.23$\pm$0.28 & 66.32$\pm$0.11 & 79.11$\pm$0.09\\
			&Mixup~{\scriptsize\textcolor{darkgray}{[NeurIPS18]}} \cite{zhang2018mixup}  & 5.30$\pm$2.02 & 90.79$\pm$2.64 &29.68$\pm$3.26 &96.71$\pm$0.05 && 46.91$\pm$2.43  & 87.61$\pm$0.46 &56.05$\pm$2.50 &82.51$\pm$0.18\\
			&RegMixup~{\scriptsize\textcolor{darkgray}{[NeurIPS22]}} \cite{pinto2022regmixup}&3.36$\pm$0.27 &92.31$\pm$0.34  &37.48$\pm$4.96 &97.10$\pm$0.14 &&40.36$\pm$1.71 & 88.33$\pm$0.35 &56.44$\pm$0.95  &82.50$\pm$0.30 \\
			\cmidrule(lr){2-11}
			&OpenMix (ours) &\bftab{2.32$\pm$0.15} &\bftab{94.81$\pm$0.34}  &\bftab{22.08$\pm$1.86} &\bftab{97.16$\pm$0.10} &&\bftab{39.61$\pm$0.54} &\bftab{89.06$\pm$0.11}  &\bftab{55.00$\pm$1.29} &\bftab{82.63$\pm$0.06}\\
			\midrule
			\midrule
			\multirow{9}{*}{DenseNet} 
			&MSP~{\scriptsize\textcolor{darkgray}{[ICLR17]}} \cite{hendrycks2017baseline} &5.66$\pm$0.45 &93.14$\pm$0.65 &38.64$\pm$4.70 &94.78$\pm$0.16 &&66.11$\pm$1.56 &86.20$\pm$0.04 &62.79$\pm$0.83 &76.96$\pm$0.20 \\
			&Doctor~{\scriptsize\textcolor{darkgray}{[NeurIPS21]}} \cite{granese2021doctor}  &5.64$\pm$0.45 &93.19$\pm0.63$  &38.29$\pm4.90$ &94.78$\pm$0.16 &&67.45$\pm1.34$  &86.30$\pm0.05$ &63.47$\pm0.34$ &76.96$\pm$0.20 \\
			&ODIN~{\scriptsize\textcolor{darkgray}{[ICLR18]}} \cite{LiangLS18} &15.37$\pm$1.98  & 82.02$\pm$2.22 & 61.77$\pm$3.53 & 94.78$\pm$0.16 &&110.50$\pm$5.09 & 75.71$\pm$0.72 & 76.37$\pm$0.89 & 76.96$\pm$0.20\\
			&Energy~{\scriptsize\textcolor{darkgray}{[NeurIPS20]}}  \cite{liu2020energy}  &8.60$\pm$0.84 & 89.21$\pm$1.18 & 51.31$\pm$2.69 & 94.78$\pm$0.16&&100.13$\pm$3.47 & 78.03$\pm$0.55 & 74.46$\pm$0.65 & 76.96$\pm$0.20\\
			&MaxLogit~{\scriptsize\textcolor{darkgray}{[ICML22]}} \cite{hendrycks2019anomalyseg}  &8.38$\pm$0.81 & 89.57$\pm$1.15  & 48.96$\pm$2.48 & 94.78$\pm$0.16 &&96.69$\pm$3.26 & 79.14$\pm$0.49 & 70.52$\pm$0.57 & 76.96$\pm$0.20\\
			&LogitNorm~{\scriptsize\textcolor{darkgray}{[ICML22]}} \cite{wei2022logitnorm} & 10.89$\pm$0.71 & 88.70$\pm$0.27 & 56.59$\pm$3.07 & 93.59$\pm$0.34 &&  116.35$\pm$3.22 & 78.14$\pm$0.60 & 74.81$\pm$0.89 & 73.13$\pm$0.48\\
			&Mixup~{\scriptsize\textcolor{darkgray}{[NeurIPS18]}} \cite{zhang2018mixup}  & 9.55$\pm$0.19 & 89.87$\pm$0.47 &37.21$\pm$1.09 &94.92$\pm$0.08 &&63.76$\pm$3.28 & 86.09$\pm$0.81 & 63.94$\pm$2.86 &77.82$\pm$0.42 \\
			&RegMixup~{\scriptsize\textcolor{darkgray}{[NeurIPS22]}} \cite{pinto2022regmixup}  &5.20$\pm$0.45 &92.02$\pm$0.95  &41.50$\pm$3.45 &95.50$\pm$0.03 &&55.81$\pm$1.40 &87.14$\pm$0.22 &63.98$\pm$1.36 &78.68$\pm$0.45\\
			\cmidrule(lr){2-11}
			&OpenMix (ours) &\bftab{4.68$\pm$0.72} &\bftab{93.57$\pm$0.81}  &\bftab{33.57$\pm$3.70} &\bftab{95.51$\pm$0.23} &&\bftab{53.83$\pm$0.93} &\bftab{87.45$\pm$0.18} &\bftab{62.22$\pm$1.15}  &\bftab{78.97$\pm$0.31}\\
			\bottomrule
		\end{tabular}
	}
\end{table*}

\begin{table}[t]
	\caption{Comparison with other methods using VGG-16. Results with ``$^{\dagger}$'' are from \cite{corbiere2021confidence}. E-AURC is also reported following \cite{corbiere2021confidence}.}
	\vskip -0.22in
	\label{table-3}
	\begin{center}
		\renewcommand\tabcolsep{5.2pt}
		\begin{small}
			\newcommand{\tabincell}[2]{\begin{tabular}{@{}#1@{}}#2\end{tabular}}
			\scalebox{0.7}{
				\renewcommand{\arraystretch}{0.9}
				\begin{tabular}{lcccc}
					\toprule
					\textbf{Method}&
					\textbf{AURC} $\downarrow$ & \textbf{E-AURC} $\downarrow$&
					\textbf{FPR95} $\downarrow$&\textbf{AUROC} $\uparrow$\\
					\midrule
					& \multicolumn{4}{c}{CIFAR-10}\\
					\cmidrule(lr){2-5}
					MSP~{\scriptsize\textcolor{darkgray}{[ICLR17]}}\cite{hendrycks2017baseline} $^{\dagger}$ &12.66$\pm$0.61 &8.71$\pm$0.50 &49.19$\pm$1.42 &91.18$\pm$0.32\\
					MCDropout~{\scriptsize\textcolor{darkgray}{[ICML16]}}\cite{GalG16}  $^{\dagger}$ &13.31$\pm$2.63 &9.46$\pm$2.41 &49.67$\pm$2.66 &90.70$\pm$1.96\\
					TrustScore~{\scriptsize\textcolor{darkgray}{[NeurIPS18]}}\cite{Jiang2018ToTO}  $^{\dagger}$ &17.97$\pm$0.45 &14.02$\pm$0.34 &54.37$\pm$1.96 &87.87$\pm$0.41 \\
					TCP~{\scriptsize\textcolor{darkgray}{[TPAMI21]}}\cite{corbiere2021confidence} $^{\dagger}$ &11.78$\pm$0.58 &7.88$\pm$0.44 &45.08$\pm$1.58 &92.05$\pm$0.34\\
					SS~{\scriptsize\textcolor{darkgray}{[NeurIPS21]}}\cite{luo2021learning}   &- &-  &44.69 &92.22 \\
					OpenMix (ours) &\bftab{6.31$\pm$0.18} &\bftab{4.41$\pm$0.15} &\bftab{38.48$\pm$1.30} &\bftab{93.56$\pm$0.26} \\
					\midrule
					& \multicolumn{4}{c}{CIFAR-100}\\
					\cmidrule(lr){2-5}
					MSP~{\scriptsize\textcolor{darkgray}{[ICLR17]}}\cite{hendrycks2017baseline}$^{\dagger}$ &113.23$\pm$2.98 &51.93$\pm$1.20 &66.55$\pm$1.56 &85.85$\pm$0.14\\
					MCDropout~{\scriptsize\textcolor{darkgray}{[ICML16]}}\cite{GalG16}$^{\dagger}$ &101.41$\pm$3.45 &46.45$\pm$1.91 &63.25$\pm$0.66 &86.71$\pm$0.30\\
					TrustScore~{\scriptsize\textcolor{darkgray}{[NeurIPS18]}}\cite{Jiang2018ToTO} $^{\dagger}$ &119.41$\pm$2.94 &58.10$\pm$1.09 &71.90$\pm$0.93  &84.41$\pm$0.15\\
					TCP~{\scriptsize\textcolor{darkgray}{[TPAMI21]}}\cite{corbiere2021confidence}$^{\dagger}$ &108.46$\pm$2.62 &47.15$\pm$0.95 &62.70$\pm$1.04 &87.17$\pm$0.21\\
					OpenMix (ours) &\bftab{73.44$\pm$0.65} &\bftab{36.41$\pm$0.45} &\bftab{61.58$\pm$0.94} &\bftab{87.47$\pm$0.12}\\
					\bottomrule
			\end{tabular}}
		\end{small}
	\end{center}
	\vskip -0.35in
\end{table}

\setParDis\noindent\textbf{Training configuration.} All models are trained using SGD with a momentum of 0.9, an initial learning rate of 0.1, and a weight decay of 5e-4 for 200 epochs with the mini-batch size of 128 for CIFAR. The learning rate is reduced by a factor of 10 at 100, and 150 epochs. For experiments on ImageNet, we perform the automatic mixed precision training. Implementation details are provided in Supp.M.

\noindent\textbf{Evaluation metrics.} \ding{172} \textbf{AURC.} The area under the risk-coverage curve (AURC) \cite{GeifmanE17} depicts the error rate computed by using samples whose confidence is higher than some confidence thresholds. \ding{173} \textbf{AUROC.} The area under the receiver
operating characteristic curve (AUROC) \cite{davis2006relationship} depicts the relationship between true positive rate (TPR) and false positive rate (FPR). \ding{174} \textbf{FPR95.} The FPR at 95\% TPR denotes the probability that a misclassified example is predicted as a correct one when the TPR is as high as $95\%$. \ding{175} \textbf{ACC.} Test accuracy (ACC) is also an important metric. \setParDef

\begin{figure}[h]
	\begin{center}
		\vskip -0.1 in
		\centerline{\includegraphics[width=1.02\columnwidth]{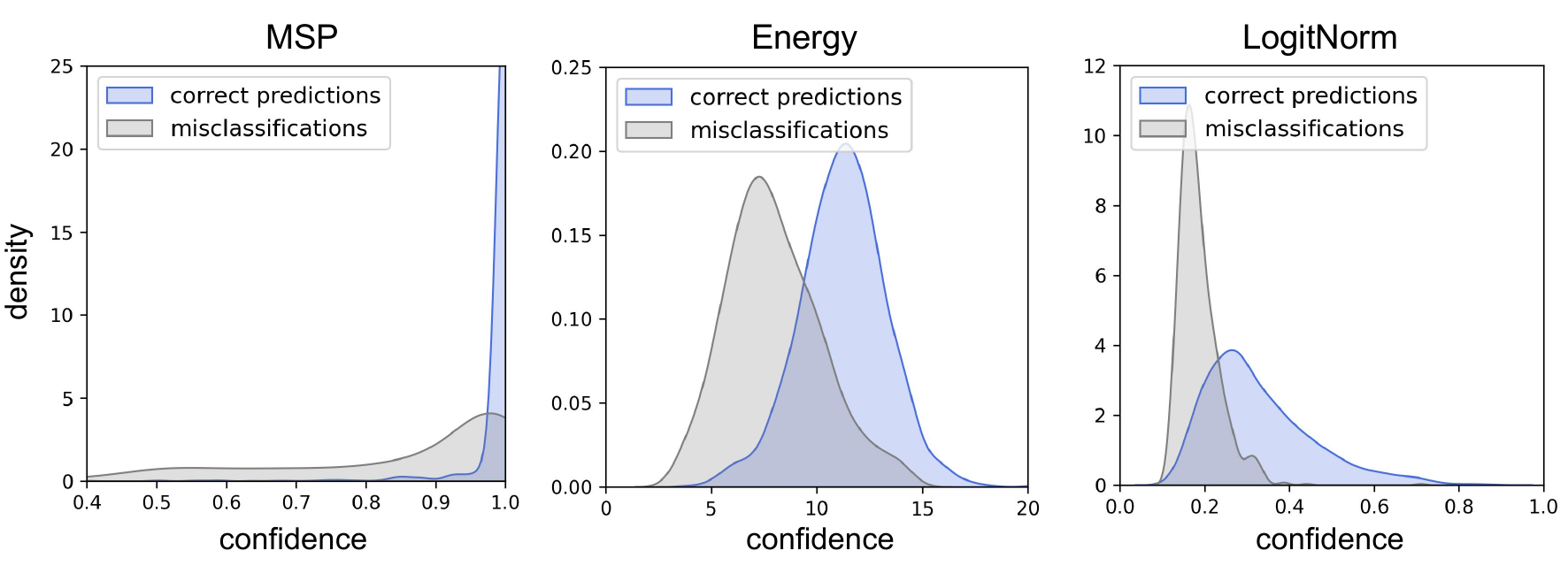}}
		\vskip -0.1 in
		\caption{OOD detection methods lead to worse confidence separation between correct and wrong samples.}
		\label{figure-5}
	\end{center}
	\vskip -0.35 in
\end{figure}
\subsection{Comparative Results}
\noindent\textbf{OOD detection methods failed in detecting misclassification errors.} As shown in Table~\ref{table-2}, we observe that the simple MSP can consistently outperform Energy \cite{liu2020energy}, MaxLogit \cite{hendrycks2019anomalyseg}, ODIN \cite{LiangLS18} and LogitNorm \cite{wei2022logitnorm}, which are strong OOD detection methods. The illustration in Fig.~\ref{figure-5} shows that those methods lead to more overlap between misclassified and correct ID data compared with MSP.
This is surprising and undesirable because in practice both OOD and misclassified samples result in significant loss, and therefore should be rejected and handed over to humans. 
This observation points out an interesting future research direction of developing confidence estimation methods that consider OOD detection and MisD in a unified manner. 

\setParDis \noindent\textbf{OpenMix improves the reliability of confidence.} \ding{172} \emph{Comparison with MSP}. The results in Table~\ref{table-2} show that OpenMix widely outperforms the strong baseline MSP. For instance, compared with MSP, ours successfully reduces the FPR95 from 30.14\% to 22.08\% under the CIFAR-10/WRNet setting. \ding{173} \emph{Comparison with Mixup variants}. We compare OpenMix with the original Mixup \cite{zhang2018mixup} and its recently developed variant RegMixup \cite{pinto2022regmixup}. We can find that they can also be outperformed by OpenMix. \ding{174} \emph{Comparison with TCP and other methods}. Since TCP \cite{corbiere2019addressing} is based on misclassified training samples, it can not be used for models with high training accuracy. Therefore, we make comparison on VGG-16 \cite{SimonyanZ14a}. In Table~\ref{table-3}, OpenMix outperforms TCP, SS \cite{luo2021learning}, MCDropout \cite{GalG16} and TrustScore \cite{Jiang2018ToTO}. 

\begin{figure}[t]
	\begin{center}
		\vskip -0.02 in
		\centerline{\includegraphics[width=\columnwidth]{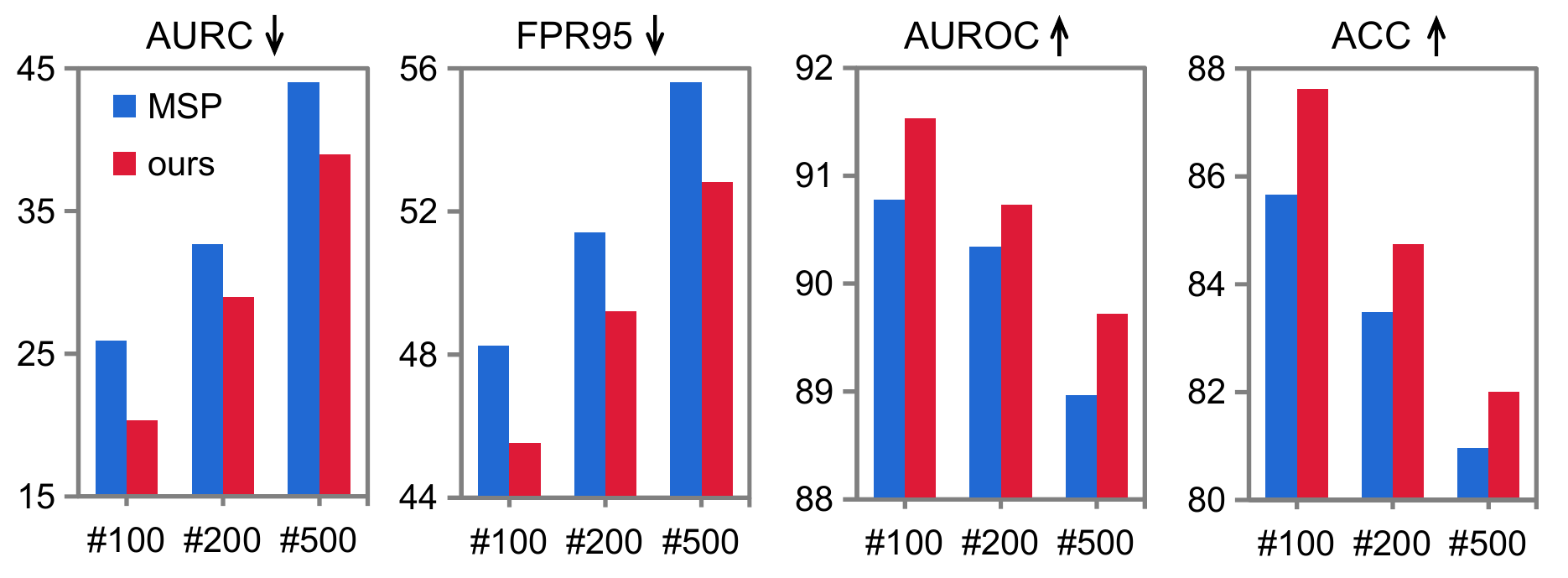}}
		\vskip -0.15 in
		\caption{Large-scale experiments on ImageNet.}
		\label{figure-6}
	\end{center}
	\vskip -0.5 in
\end{figure}

\setParDis \noindent\textbf{Large-scale experiments on ImageNet.} To demonstrate the scalability of our method, in Fig.~\ref{figure-6}, we report the results on ImageNet. Specifically, three settings which consist of random 100, 200, and 500 classes from ImageNet are conducted. For each experiment, we randomly sample another set of disjoint classes from ImageNet as outliers. As can be seen, OpenMix consistently boosts the MisD performance of baseline, improving the confidence reliability remarkably. Detailed training setups are provided in Supp.M.

\begin{figure}[t]
	\begin{center}
		\centerline{\includegraphics[width=\columnwidth]{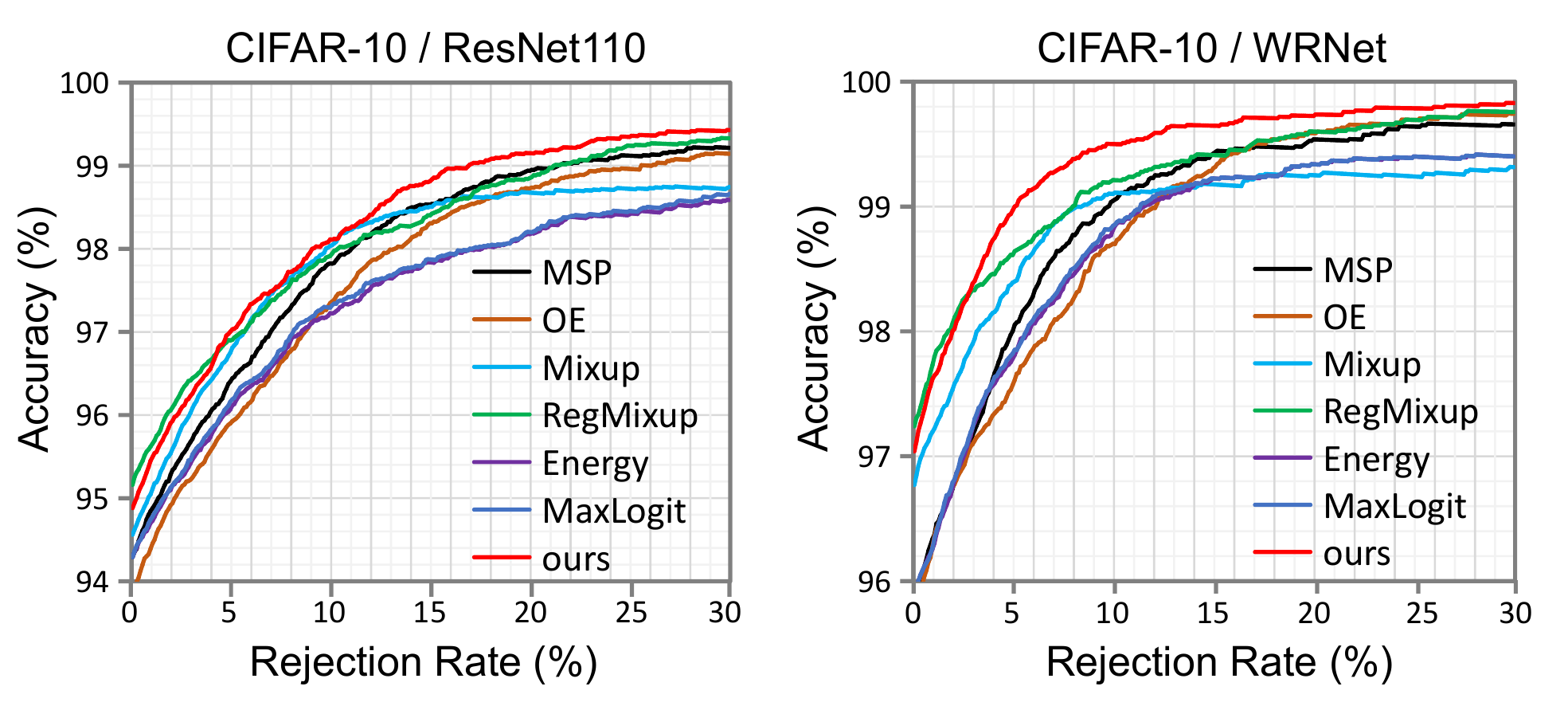}}
		\vskip -0.1 in
		\caption{Accuracy-rejection curves analysis: \ding{172} \textbf{\emph{Diverging}} between OOD detection methods (OE, Energy, MaxLogit) and MSP. \ding{173} \textbf{\emph{Crossing-over}} between Mixup/RegMixup and MSP. \ding{174} \textbf{\emph{Evenly spaced}} between our method and MSP.}
		\label{figure-7}
	\end{center}
	\vskip -0.35 in
\end{figure}
\noindent\textbf{Further analysis on accuracy-rejection curves.} Fig.~\ref{figure-7} plots the accuracy against rejection rate, \emph{i.e.}, accuracy-rejection curve (ARC) \cite{nadeem2009accuracy}, to straightway and graphically make comparison among several models. Particularly, we identify three different types of relationships described in \cite{nadeem2009accuracy}, \emph{i.e.}, \emph{diverging}, \emph{crossing-over}, and \emph{evenly spaced}. 
For selection of the best model by ARCs, \ding{172} if the desired accuracy is known,
one can move horizontally across the ARC plot and select the model with the lowest rejection
rate. \ding{173} Conversely, if the acceptable rejection rate is known, we select the model with the
highest accuracy. The results in Fig.~\ref{figure-7} recommend our method as the best in both cases. 
\setParDef

\begin{table}[h]
	\vskip 0.02in
	\caption{Ablation Study of each component in our method.}
	\vskip -0.1in
	\label{table-4}
	\setlength\tabcolsep{3.2pt}
	\centering
	\renewcommand{\arraystretch}{0.93}
	\scalebox{0.7}{
		\begin{tabular}{lccccccccc}
			\toprule \multirow{2}{*}{\textbf{Method}} & \multicolumn{4}{c}{\textbf{CIFAR-10}} && \multicolumn{4}{c}{\textbf{CIFAR-100}}\\
			\cmidrule(lr){2-5} \cmidrule(lr){7-10}
			&\textbf{AURC}  & \textbf{AUROC}  & \textbf{FPR95}   &\textbf{ACC}   &&\textbf{AURC}  & \textbf{AUROC}  & \textbf{FPR95}   &\textbf{ACC} \\
			\midrule
			MSP   &9.52 &90.13 & 43.33 & 94.30 &&89.05 &84.91 &65.65 &73.30\\
			+ RC     & 9.55  & 91.15 & 40.03 & 94.02 &  & 94.31 & 85.53 & 65.78 & 71.44 \\
			+ OT     & 12.38 & 87.13 & 61.83 & 93.84 &  & 99.86 & 82.51 & 72.94 & 72.62 \\
			OpenMix &\bftab{6.31} &\bftab{92.09}  &\bftab{39.63} &\bftab{94.98} &&\bftab{73.84} &\bftab{85.83} &\bftab{64.22} &\bftab{75.77}\\
			\bottomrule	
		\end{tabular}
	}
	\vskip -0.1 in
\end{table}
\subsection{Ablation Study}
\noindent\textbf{The effect of each component of OpenMix.} Our method is comprised of two components: \emph{learning with reject class} (\textbf{RC})  and \emph{outlier transformation} (\textbf{OT}). \ding{172} With only RC, the original outlier samples are used and labeled as the $k+1$ class. \ding{173} With only OT, it is reasonable to assign the following soft label to the mixed data: $\mathbb I^{\breve{y}} = \lambda \mathbb I^{y} + (1-\lambda) \mathcal{U}$. From Table~\ref{table-4}, we have three key observations: Firstly, RC performs slightly better or comparable with MSP, indicating that directly mapping OOD outliers to a reject class offers limited help. Secondly, OT alone can observably harm the performance. We expect this is because the interpolation between ID labels and uniform distribution suffers from the same issue as OE. Thirdly, OpenMix integrates them in a unified and complementary manner, leading to significant and consistent improvement over baseline. Supp.M provides more results on WRNet and DenseNet.

\setParDis \noindent\textbf{The choices of outlier dataset.}  Fig.~\ref{figure-8} reports results of using different outlier datasets. First, we can observe that using simple noises like {\ttfamily Gaussian noise} in OpenMix can lead to notable improvement. This verifies our insight that exposing low density regions is beneficial for MisD. Secondly, real-world datasets with semantic information yield better performance. 
\begin{figure}[h]
	\begin{center}
		\vskip -0.17 in
		\centerline{\includegraphics[width=\columnwidth]{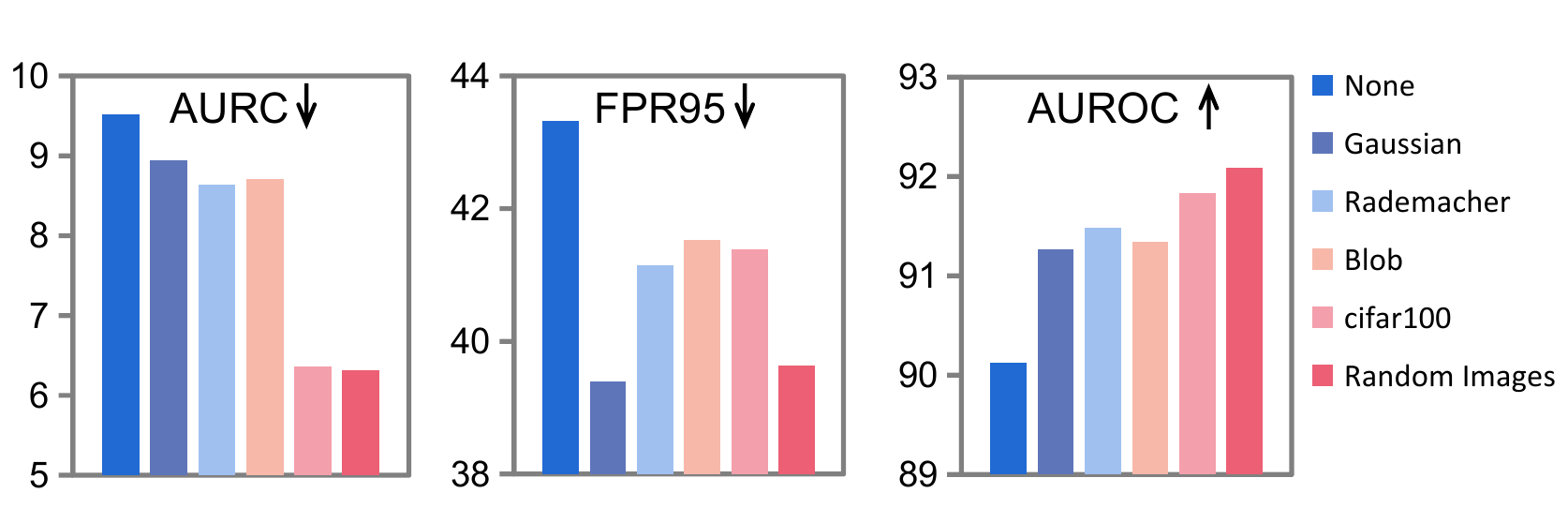}}
		\vskip -0.1 in
		\caption{Ablation study on the effect of different outlier datasets.}
		\label{figure-8}
	\end{center}
	\vskip -0.4 in
\end{figure}

\setParDis \noindent\textbf{Comparison of different interpolation strategies.} We use Mixup for outlier transformation due to its simplicity. Table~\ref{table-55} (CIFAR/ResNet110) shows that CutMix \cite{yun2019cutmix} and Manifold Mixup \cite{verma2019manifold} are also effective, further improving the performance of OpenMix.
\setParDef

\begin{table}[h]
	\caption{Comparison of different interpolation strategies.}
	\vskip -0.1in
	\label{table-55}
	\setlength\tabcolsep{2pt}
	\centering
	\renewcommand{\arraystretch}{1}
	\scalebox{0.7}{
		\begin{tabular}{lcccccccc}
			\toprule
			\multirow{2}{*}{Method} & \multicolumn{4}{c}{CIFAR-10} & \multicolumn{4}{c}{CIFAR-100} \\
			\cmidrule(lr){2-5} \cmidrule(lr){6-9}
			&\textbf{AURC}  & \textbf{AUROC}  & \textbf{FPR95}   &\textbf{ACC}   &\textbf{AURC}  & \textbf{AUROC}  & \textbf{FPR95}   &\textbf{ACC} \\
			\midrule
			ours w/Mixup &6.31 &92.09  &39.63 &94.98 &73.84 &85.83 &\textbf{64.22} &75.77  \\
			ours w/CutMix & 6.74 &\textbf{93.45}  &\textbf{36.82} &93.73 &76.28 &\textbf{86.49} &64.78 &74.15  \\
			ours w/Manifold & \textbf{5.67} &92.46  &36.91  &\textbf{95.21} &\textbf{71.71}  &85.83  &66.11  &\textbf{76.13}  \\
			\bottomrule
		\end{tabular}
	}
	\vskip -0.1in
\end{table}

\begin{table*}[t]
	\caption{Integrating OpenMix with CRL \cite{MoonKSH20} and FMFP \cite{zhu2022rethinking}. Our method can remarkably improve their MisD performance on CIFAR-10.}
	\vskip -0.1in
	\label{table-5}
	\setlength\tabcolsep{4.9pt}
	\centering
	\renewcommand{\arraystretch}{1}
	\scalebox{0.65}{
		\begin{tabular}{lccccccccccccccc}
			\toprule
		\multirow{2}{*}{\textbf{Method}} & \multicolumn{3}{c}{\textbf{AURC} $\downarrow$} && \multicolumn{3}{c}{\textbf{AUROC} $\uparrow$} && \multicolumn{3}{c}{\textbf{FPR95} $\downarrow$} && \multicolumn{3}{c}{\textbf{ACC} $\uparrow$}\\
		\cmidrule(lr){2-4} \cmidrule(lr){6-8} \cmidrule(lr){10-12} \cmidrule(lr){13-16}
		&ResNet110 & WRNet & DenseNet  && ResNet110 & WRNet & DenseNet  && ResNet110 & WRNet & DenseNet && ResNet110 & WRNet & DenseNet\\
		\midrule
		CRL \cite{MoonKSH20} &6.60$\pm$0.12 &3.99$\pm$0.17  &5.71$\pm$0.24 && 93.59$\pm$0.05 &94.37$\pm$0.21 &93.70$\pm$0.14  &&41.00$\pm$0.28 &32.83$\pm$1.17   &39.03$\pm$0.69 &&93.63$\pm$0.08 &95.42$\pm$0.20 &94.33$\pm$0.13  \\
		+ ours  &\bftab{4.48$\pm$0.10} &\bftab{2.02$\pm$0.05}  &\bftab{4.02$\pm$0.38} &&\bftab{94.43$\pm$0.02} &\bftab{95.40$\pm$0.11}  &\bftab{94.48$\pm$0.51} &&\bftab{33.20$\pm$0.43} &\bftab{25.50$\pm$0.93}  &\bftab{44.43$\pm$2.21}  &&\bftab{94.85$\pm$0.10} &\bftab{96.95$\pm$0.07}  &\bftab{95.36$\pm$0.09}\\
		\midrule

		FMFP \cite{zhu2022rethinking}  &5.33$\pm$0.15 &2.28$\pm$0.03   &4.09$\pm$0.11 &&94.07$\pm$0.09 &95.71$\pm$0.12 &94.82$\pm$0.10  &&39.37$\pm$0.77 & 25.20$\pm$1.23   & 30.35$\pm$1.72 &&94.36$\pm$0.09 &96.55$\pm$0.08 &95.11$\pm$0.16  \\
		+ ours  &\bftab{3.94$\pm$0.11} &\bftab{1.70$\pm$0.12}  &\bftab{3.58$\pm$0.16} &&\bftab{94.32$\pm$0.10} &\bftab{95.90$\pm$0.13} &\bftab{94.64$\pm$0.12} &&\bftab{30.41$\pm$0.83} &\bftab{18.78$\pm$2.13} &\bftab{29.36$\pm$0.68}  &&\bftab{95.43$\pm$0.08} &\bftab{97.33$\pm$0.11}  &\bftab{95.71$\pm$0.13}\\
			\bottomrule
		\end{tabular}
	}
	\vskip -0.07 in
\end{table*}

\begin{figure}[!h]
	\begin{center}
		\centerline{\includegraphics[width=1.04\columnwidth]{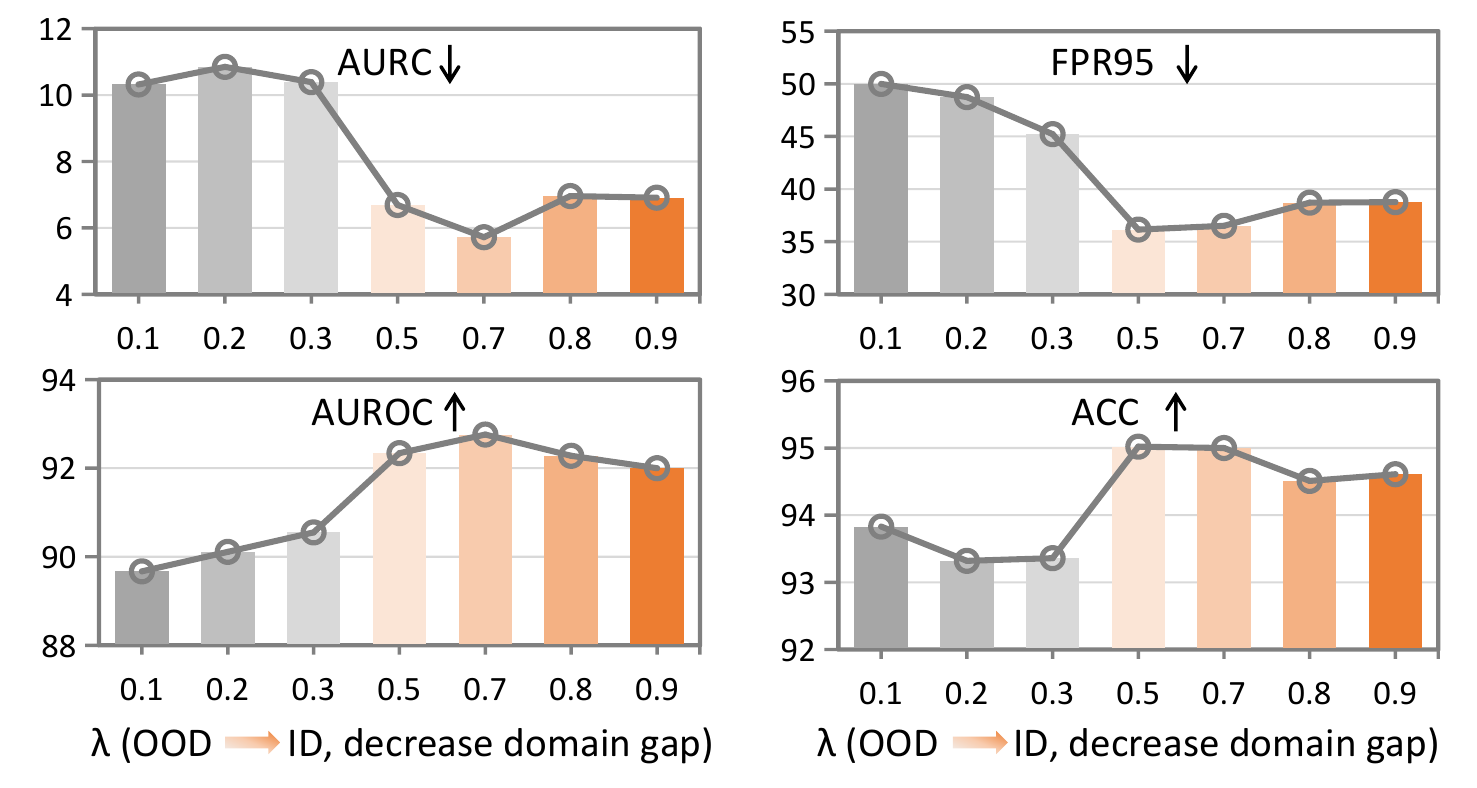}}
		\vskip -0.12 in
		\caption{Relationship between domain gap and performance gain. CIFAR-10/ResNet110, the used outlier dataset is {\ttfamily RandImages}.}
		\label{figure-99}
	\end{center}
	\vskip -0.3 in
\end{figure}
\subsection{Further Experiments and Analysis}
\noindent\textbf{The relationship between domain gap and performance gain.} Given a specific outlier dataset, the proposed outlier transformation can control and adjust the domain gap flexibly: if the outlier set is far, we can increase the ID information by enlarging $\lambda$
in Eq.~\ref{eq4}, and vice versa. In Fig.~\ref{figure-99}, we can observe that decreasing the domain gap firstly increases the performance gain and then reduces the gain.

\setParDis\noindent\textbf{OpenMix improves CRL and FMFP.} CRL \cite{MoonKSH20} ranks the confidence to reflect the historical correct rate. FMFP \cite{zhu2022rethinking} improves confidence reliability by seeking flat minima. Differently, OpenMix focuses on the complementary strategy to exploit the unlabeled outlier data. We show in Table~\ref{table-5} that our method can consistently boost the performance of CRL and FMFP, demonstrating the complementarity effectiveness of OpenMix.

\begin{figure}[h]
	\begin{center}
		\vskip -0.1 in
		\centerline{\includegraphics[width=\columnwidth]{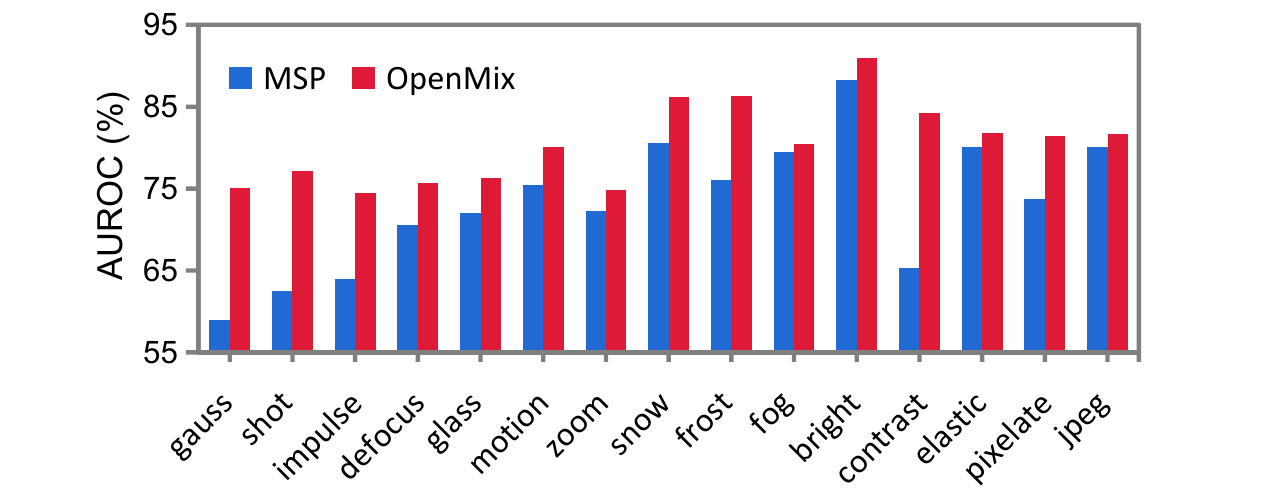}}
		\vskip -0.1 in
		\caption{MisD under distribution shift. Performance on 15 types of corruption under the severity level of 5 is reported. The model is trained on CIFAR-10/ResNet110 and tested on C10-C \cite{HendrycksD19}.}
		\label{figure-9}
	\end{center}
	\vskip -0.3 in
\end{figure}
\noindent\textbf{MisD under distribution shift.} In practice, environments can be easily changed, \emph{e.g.}, weather change from sunny to cloudy then to rainy. The model still needs to make reliable decisions under such distribution or domain shift conditions. To mimic those scenarios, we test the model on corruption datasets like C10-C \cite{HendrycksD19}. 
Fig.~\ref{figure-9} shows that OpenMix significantly improves the MisD performance under various corruptions, and the averaged AUROC can be improved from 73.28\% to 80.44\%. Supp.M provides averaged results under different severity levels and corruptions on C10-C and C100-C.

\begin{table}[h]
	\vskip 0.02in
	\caption{OpenMix improves MisD in long-tailed recognition.}
	\vskip -0.1in
	\label{table-6}
	\setlength\tabcolsep{4.1pt}
	\centering
	\renewcommand{\arraystretch}{}
	\scalebox{0.673}{
		\begin{tabular}{lccccccccc}
			\toprule
			\multirow{2}{*}{\textbf{Method}} & \multicolumn{4}{c}{\textbf{CIFAR-10-LT}} && \multicolumn{4}{c}{\textbf{CIFAR-100-LT}}\\
			\cmidrule(lr){2-5} \cmidrule(lr){7-10}
			&\textbf{AURC}  & \textbf{AUROC}  & \textbf{FPR95}   &\textbf{ACC}   &&\textbf{AURC}  & \textbf{AUROC}  & \textbf{FPR95}   &\textbf{ACC} \\
			\midrule
			LA \cite{MenonJRJVK21}  & 62.13  & 84.52 & 69.77  & 79.02 & & 347.43 & 78.46 & 76.47 & 41.69\\
			+ CRL  & 63.81 & 85.30  & 63.05  & 78.50 &&345.05 & 78.74 & 76.19  & 41.58\\
			+ ours &\bftab{38.07} & \bftab{87.21} &\bftab{64.14}  & \bftab{83.60} &&\bftab{284.77} &\bftab{81.22} &\bftab{73.80}  &\bftab{46.52}\\
			\midrule
			VS \cite{KiniPOT21}  & 58.45 & 84.47 & 70.15 & 80.11 &&343.48 &78.20 & 77.25 & 42.20 \\
			+ CRL   & 62.06  & 83.98 & 67.19 & 79.69 &&345.06 & 78.29 & 77.44 & 41.88 \\
			+ ours  &\bftab{41.52} &\bftab{87.12} &\bftab{63.31} &\bftab{83.02} &&\bftab{277.34} & \bftab{81.42} &\bftab{72.93} &\bftab{47.16} \\
			\bottomrule	
		\end{tabular}	
	}
\end{table}
\noindent\textbf{MisD in long-tailed recognition.} The class distributions in real-world settings often follow a long-tailed distribution \cite{cao2019learning, MenonJRJVK21}. For example, in a disease diagnosis system, the normal samples are typically more than the disease samples. In such failure-sensitive applications, reliable confidence estimation is especially crucial. We use long-tailed classification datasets CIFAR-10-LT and CIFAR-100-LT \cite{cao2019learning} with an imbalance ratio $\rho = 100$. The network is ResNet-32. We built our method on two long-tailed recognition methods LA \cite{MenonJRJVK21} and VS \cite{KiniPOT21}. Table~\ref{table-6} shows our method remarkably improves MisD performance and long-tailed classification accuracy. More results can be found in Supp.M.

\noindent\textbf{OpenMix improves OOD detection.} A good confidence estimator should help separate both the OOD and misclassified ID samples from correct predictions. Therefore, besides MisD, we explore the OOD detection ability of our method. The ID dataset is CIFAR-10. For the OOD datasets, we follow recent works that use \textbf{\emph{six}} common benchmarks: Textures, SVHN, Place365, LSUN-C, LSUN-R and iSUN. Metrics are AUROC, AUPR and FPR95 \cite{hendrycks2017baseline}. Table~\ref{table-7} shows that OpenMix also achieves strong OOD detection performance along with high MisD ability, which is not achievable with OE and other OOD detection methods. Results on CIFAR-100 can be found in Supp.M. \setParDef

\begin{table}[h]
	\caption{OOD detection performance. All values are percentages and are averaged over \textbf{\emph{six}} OOD test datasets.}
	\vskip -0.1in
	\label{table-7}
	\setlength\tabcolsep{2.5pt}
	\centering
	\renewcommand{\arraystretch}{1.1}
	\scalebox{0.6}{
		\begin{tabular}{lccccccccc}
			\toprule
			\multirow{2}{*}{\textbf{Method}} & \multicolumn{3}{c}{\textbf{FPR95} $\downarrow$} & \multicolumn{3}{c}{\textbf{AUROC} $\uparrow$} & \multicolumn{3}{c}{\textbf{AUPR} $\uparrow$}\\
			\cmidrule(lr){2-4} \cmidrule(lr){5-7} \cmidrule(lr){8-10}
			& ResNet & WRN & DenseNet & ResNet & WRN & DenseNet & ResNet & WRN & DenseNet\\\midrule
			MSP \cite{hendrycks2017baseline}  &51.69 &40.83 &48.60 &89.85 &92.32 &91.55 &97.42&97.93 &98.11 \\
			LogitNorm \cite{wei2022logitnorm} &29.72 &12.97 &19.72 &94.29 &97.47 &96.19 &98.70 &99.47 &99.11 \\
			ODIN \cite{LiangLS18} &35.04 &26.94 &30.67 &91.09 &93.35 &93.40 &97.47 &97.98 &98.30 \\
			Energy \cite{liu2020energy} &33.98 &25.48 &30.01 &91.15 &93.58 &93.45 &97.49 &98.00 &98.35 \\
			MaxLogit \cite{hendrycks2019anomalyseg}  &34.61 &26.72 &30.99 &91.13 &93.14 &93.44 &97.46 &97.78 &98.35 \\
			OE \cite{hendrycks2019deep} &\bftab{5.28} &\bftab{3.49} &\bftab{5.25} &\bftab{98.04} &\bftab{98.59} &\bftab{98.20} &\bftab{99.55} &\bftab{99.71} &\bftab{99.62} \\
			FMFP \cite{zhu2022rethinking}  &39.50 &26.83 &35.12 &93.83 &96.22 &94.88 &98.73 &99.23 &98.95 \\
			OpenMix (ours) &39.72 &16.86 &32.75 &93.22 &96.92 &94.85 &98.46 &99.34 &98.84  \\
			\bottomrule
		\end{tabular}
	}
	\vskip -0.13in
\end{table}

\section{Conclusive Remarks}
MisD is an important but under-explored area of research. In this paper, we propose OpenMix, a simple yet effective approach that explores outlier data for helping detect misclassification errors. Extensive experiments demonstrate that OpenMix significantly improves the confidence reliability of DNNs and yields strong performance under distribution shift and long-tailed scenarios.
Particularly, recent works \cite{kim2021unified, jaeger2022call, cen2023devil, bernhardtfailure} claim that none of the existing methods performs well for both OOD detection and MisD. Fortunately, the proposed OpenMix can detect OOD and misclassified samples in a unified manner. We hope that our work opens possibilities to
explore unified methods that can detect both OOD samples and misclassified samples.

\setParDis
\noindent
\textbf{Acknowledgement.} This work has been supported by the National Key Research and Development Program (2018AAA0100400), National Natural Science Foundation of China (61836014, 62222609, 62076236, 61721004), Key Research Program of Frontier Sciences of Chinese Academy of Sciences (ZDBS-LY-7004), Youth Innovation Promotion Association of Chinese Academy of Sciences (2019141).
\setParDef

\appendix
\onecolumn
\clearpage
\begin{center}{\bf {\Large Supplementary Material}}
\end{center}
\vspace{0.1in}

\section{More Details on Experimental Setups}
\subsection{Experiments on CIFAR-10 and CIFAR-100}
\noindent
\textbf{Auxiliary datasets.}
We use {\ttfamily 300K Random Images} \cite{hendrycks2019deep} as the auxiliary outlier dataset for experiments with CIFAR-10 and CIFAR-100. Specifically, {\ttfamily 300K Random Images} is a cleaned and debiased dataset with 300K natural images. In this dataset, images that belong to CIFAR classes are removed so that in-distribution (ID) training set and outlier dataset are disjoint. In Sec. 5.2 and Fig. 8, we conduct experiments on CIFAR-10 to study the effectiveness of other outlier datasets, \emph{i.e.}, {\ttfamily Gaussian}, {\ttfamily Rademacher}, {\ttfamily Blob}, {\ttfamily CIFAR-100}. Following \cite{wei2021open}, {\ttfamily Gaussian} noises are sampled from an isotropic Gaussian distribution. {\ttfamily Rademacher} noises are sampled from a symmetric Rademacher distribution. {\ttfamily Blob} noises consist of algorithmically generated amorphous shapes with definite edges.

\setParDis
\noindent
\textbf{Hyperparameters.} For OE \cite{hendrycks2019deep}, we set $\lambda=0.5$ in Eq. 2 as recommended in the original paper \cite{hendrycks2019deep}. For Mixup \cite{zhang2018mixup}, the coefficient of linear interpolation $\lambda$ is sampled as $\lambda\sim\text{Beta}(\alpha, \alpha)$, and we set $\alpha=0.3$ as recommended in the original paper \cite{zhang2018mixup}. For RegMixup \cite{pinto2022regmixup}, we set $\alpha=10$ as recommended in the original paper \cite{pinto2022regmixup}.
For our OpenMix, the $\lambda$ in Eq. 4 is sampled as $\lambda\sim\text{Beta}(\alpha, \alpha)$, and we set $\alpha=10$ in our experiments. The $\gamma$ in Eq. 5 is set as 1.\setParDef

\subsection{Experiments on ImageNet}
\noindent For experiments on ImageNet, the backbone is ResNet-50 \cite{he2016deep} and we perform automatic mixed
precision to accelerate the training by using the open-sourced code at \url{https://github.com/NVIDIA/apex/tree/master/examples/imagenet}. For each experiment, we train 90 epochs.
Three settings which consist of random 100, 200, and 500 classes from ImageNet are conducted after shuffling the class order with the fixed random seed 1993. For each experiment, we use another set of disjoint classes from ImageNet as outliers, and the outlier dataset has the same number of classes as that of training set. The $\lambda$ in Eq. 4 is sampled as $\lambda\sim\text{Beta}(\alpha, \alpha)$, and we set $\alpha=10$. The $\gamma$ in Eq. 5 is set to be 0.5.

\begin{table}[h]
	\vskip 0.05in
	\caption{MisD performance under distribution shift. The averaged
		results for 15 kinds of corruption under five different level
		perturbation severity are reported. }
	\vskip -0.07in
	\label{table-A1}
	\setlength\tabcolsep{3pt}
	\centering
	\renewcommand{\arraystretch}{1.1}
	\scalebox{0.8}{
		\begin{tabular}{lcccccccccccc}
			\toprule
			\multirow{2}{*}{\textbf{Method}} & \multicolumn{3}{c}{\textbf{AUROC} $\uparrow$} & \multicolumn{3}{c}{\textbf{AURC} $\downarrow$} & \multicolumn{3}{c}{\textbf{FPR95} $\downarrow$} & \multicolumn{3}{c}{\textbf{ACC} $\uparrow$}\\
			\cmidrule(lr){2-4} \cmidrule(lr){5-7} \cmidrule(lr){8-10} \cmidrule(lr){11-13}
			& ResNet & WRN & DenseNet & ResNet & WRN & DenseNet & ResNet & WRN & DenseNet & ResNet & WRN & DenseNet\\ \cmidrule(lr){2-13}
			&\multicolumn{11}{c}{\textbf{CIFAR-10-C}}  \\
			\midrule
			MSP \cite{hendrycks2017baseline} & 79.92  &83.34 &81.82 &154.58 &120.36 &154.72 &70.23 &64.48 &68.56 &72.27 &75.57 &71.08  \\
			CRL \cite{MoonKSH20} & 82.57  &85.86 &83.86 &143.19 &100.27 &135.46 &68.26 &62.86 &66.93 &71.19 &76.24 &71.82  \\
			OpenMix &\bftab{84.98} &\bftab{90.38} &\bftab{85.62} &\bftab{65.51} &\bftab{27.78} &\bftab{71.86} &\bftab{62.11} &\bftab{48.07} &\bftab{60.65} &\bftab{82.03} &\bftab{88.33} &\bftab{81.38}  \\
			\bottomrule
			\toprule
			&\multicolumn{11}{c}{\textbf{CIFAR-100-C}}  \\ \cmidrule(lr){2-13}
			MSP \cite{hendrycks2017baseline} & 77.39  &79.70 &75.86 &356.87 &299.82 &376.37 &76.70 &72.77 &76.88 &45.27 &51.38 &44.92  \\
			CRL \cite{MoonKSH20} & 79.00  &80.71 &78.15 &340.48 &273.60 &346.73 &74.68 &71.13 &75.25 &45.91 &53.38 &46.56  \\
			OpenMix &\bftab{78.56} &\bftab{84.05} &\bftab{79.00} &\bftab{303.82} &\bftab{176.15} &\bftab{299.71} &\bftab{74.61} &\bftab{66.24} &\bftab{74.45} &\bftab{50.61} &\bftab{62.09} &\bftab{51.29}  \\
			\bottomrule
		\end{tabular}
	}
\end{table}
\section{Additional Experimental Results}
\subsection{More results for MisD under distribution shift}
\noindent Table~\ref{table-A1} presents more results of MisD under distribution shift. The models trained on clean datasets (CIFAR-10 and CIFAR-100) are evaluated on corrupted dataset CIFAR-10/100-C \cite{HendrycksD19}. the corruption dataset contains copies of the original validation set with 15 types of corruptions of algorithmically generated corruptions from noise, blur, weather, and digital categories. Each type of corruption has five levels of severity, resulting in 75 distinct corruptions. In Table~\ref{table-A1}, we can observe that OpenMix performs the best.

\subsection{More results for long-tailed MisD}
\noindent Besides the results in Table. 6, we also compared the MisD performance of our method with TLC \cite{Li2022CVPR} under the same experimental setup. The results of TLC and others are from \cite{Li2022CVPR}. As can be observed from Fig.~\ref{figure-A1}, our method has the best performance, \emph{i.e.}, highest AUROC and lowest FPR95.
\begin{figure}[t]
	\begin{center}
		\centerline{\includegraphics[width=0.8\columnwidth]{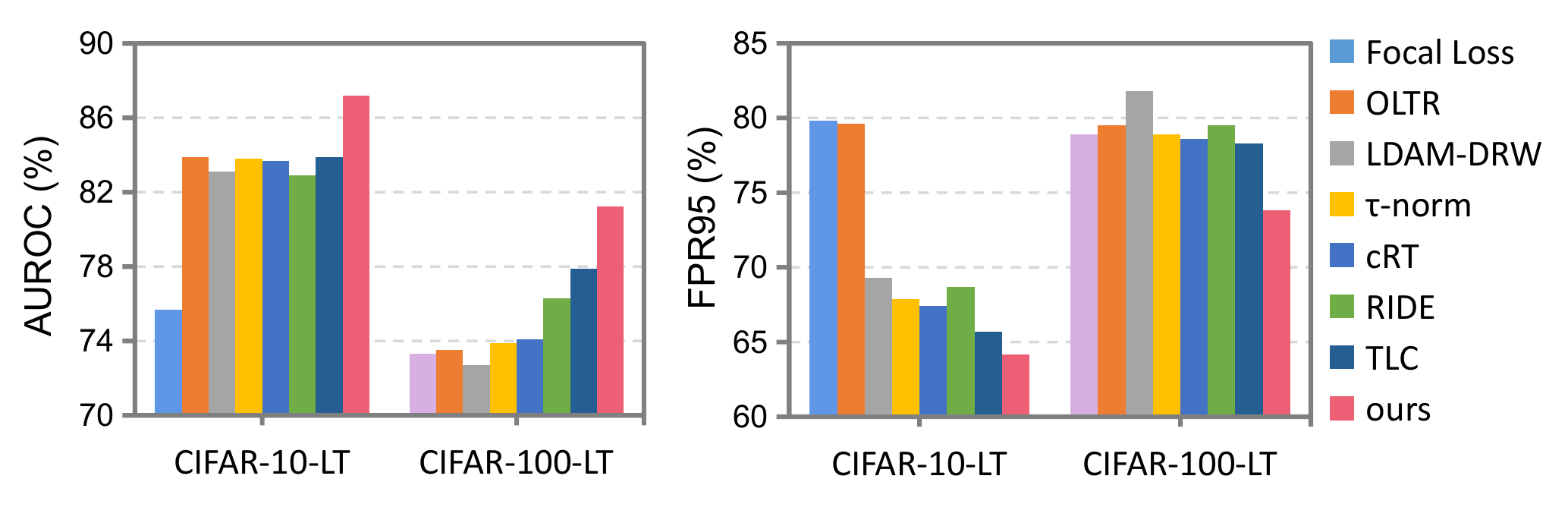}}
		\vskip -0.15 in
		\caption{More comparison on long-tailed MisD.}
		\label{figure-A1}
	\end{center}
	\vskip -0.3  in
\end{figure}

\subsection{More results for OOD detection}
\noindent Table~\ref{table-A2} presents the detailed results of OOD detection performance on CIFAR-10 and CIFAR-100. From the results, we show that our method can yield strong OOD detection performance. In addition, since Openmix is a training-time method, it can combine with any other post-processing OOD detection methods such as ODIN, Energy and MaxLogit to get higher OOD detection performance. 

\begin{table}[t]
	\caption{OOD detection performance. All values are percentages and are averaged over six OOD test datasets.}
	\vskip -0.07in
	\label{table-A2}
	\setlength\tabcolsep{4pt}
	\centering
	\renewcommand{\arraystretch}{1}
	\scalebox{0.9}{
		\begin{tabular}{lccccccccc}
			\toprule
			\multirow{2}{*}{\textbf{Method}} & \multicolumn{3}{c}{\textbf{FPR95} $\downarrow$} & \multicolumn{3}{c}{\textbf{AUROC} $\uparrow$} & \multicolumn{3}{c}{\textbf{AUPR} $\uparrow$}\\
			\cmidrule(lr){2-4} \cmidrule(lr){5-7} \cmidrule(lr){8-10}
			& ResNet & WRN & DenseNet & ResNet & WRN & DenseNet & ResNet & WRN & DenseNet\\\midrule
			&\multicolumn{9}{c}{ID: CIFAR-10} \\
			\cmidrule(lr){2-10}
			MSP \cite{hendrycks2017baseline}  &51.69 &40.83 &48.60 &89.85 &92.32 &91.55 &97.42&97.93 &98.11 \\
			LogitNorm \cite{wei2022logitnorm} &29.72 &12.97 &19.72 &94.29 &97.47 &96.19 &98.70 &99.47 &99.11 \\
			ODIN \cite{LiangLS18} &35.04 &26.94 &30.67 &91.09 &93.35 &93.40 &97.47 &97.98 &98.30 \\
			Energy \cite{liu2020energy} &33.98 &25.48 &30.01 &91.15 &93.58 &93.45 &97.49 &98.00 &98.35 \\
			MaxLogit \cite{hendrycks2019anomalyseg}  &34.61 &26.72 &30.99 &91.13 &93.14 &93.44 &97.46 &97.78 &98.35 \\
			OE \cite{hendrycks2019deep} &\bftab{5.28} &\bftab{3.49} &\bftab{5.25} &\bftab{98.04} &\bftab{98.59} &\bftab{98.20} &\bftab{99.55} &\bftab{99.71} &\bftab{99.62} \\
			CRL \cite{MoonKSH20} &51.18 &40.83 &47.28 &91.21 &93.67 &92.37 &98.11 &98.67 &98.35 \\
			FMFP \cite{zhu2022rethinking}  &39.50 &26.83 &35.12 &93.83 &96.22 &94.88 &98.73 &99.23 &98.95 \\
			OpenMix (ours) &39.72 &16.86 &32.75 &93.22 &96.92 &94.85 &98.46 &99.34 &98.84  \\
			\bottomrule
			\toprule
			&\multicolumn{9}{c}{ID: CIFAR-100}  \\
			\cmidrule(lr){2-10}
			MSP \cite{hendrycks2017baseline}  &81.68 &77.53 &77.03 &74.21 &77.96 &76.79 &93.34 &94.36 &93.94 \\
			LogitNorm \cite{wei2022logitnorm} &63.49 &57.38 &61.56 &82.50 &86.60 &82.10 &95.43 &96.80 &95.16 \\
			ODIN \cite{LiangLS18} &74.30 &76.03 &69.44 &76.55 &79.57 &80.53 &93.54 &94.59 &94.78 \\
			Energy \cite{liu2020energy} &74.42 &74.93 &68.36 &76.43 &79.89 &80.87 &93.59 &94.66 &94.86 \\
			MaxLogit \cite{hendrycks2019anomalyseg}  &74.45 &75.27 &69.85 &76.61 &79.75 &80.48 &93.66 &94.67 &94.77 \\
			OE \cite{hendrycks2019deep} &\bftab{59.85} &\bftab{49.02} &\bftab{53.03} &\bftab{86.33} &\bftab{90.07}  &\bftab{88.51} &\bftab{96.47} &\bftab{97.67} &\bftab{97.25} \\
			CRL \cite{MoonKSH20} &81.67 &79.08 &75.77 &72.72 &76.81 &76.41 &92.69 &94.22 &93.85 \\
			FMFP \cite{zhu2022rethinking} &80.19 &70.98 &72.87 &72.92 &81.54 &77.56 &92.94 &95.71 &94.19 \\
			OpenMix (ours) &74.66 &68.87 &66.63 &75.95 &84.88 &81.23 &93.56 &96.55 &95.30 \\
			\bottomrule
		\end{tabular}
	}
\end{table}

\subsection{Ablation study on different distribution of interpolation coefficient $\lambda\sim\text{Beta}(\alpha, \alpha)$}

\begin{figure}
	\begin{minipage}{.5\linewidth}
		\centering
		\vskip -0.17 in
		\includegraphics[width=6.5cm]{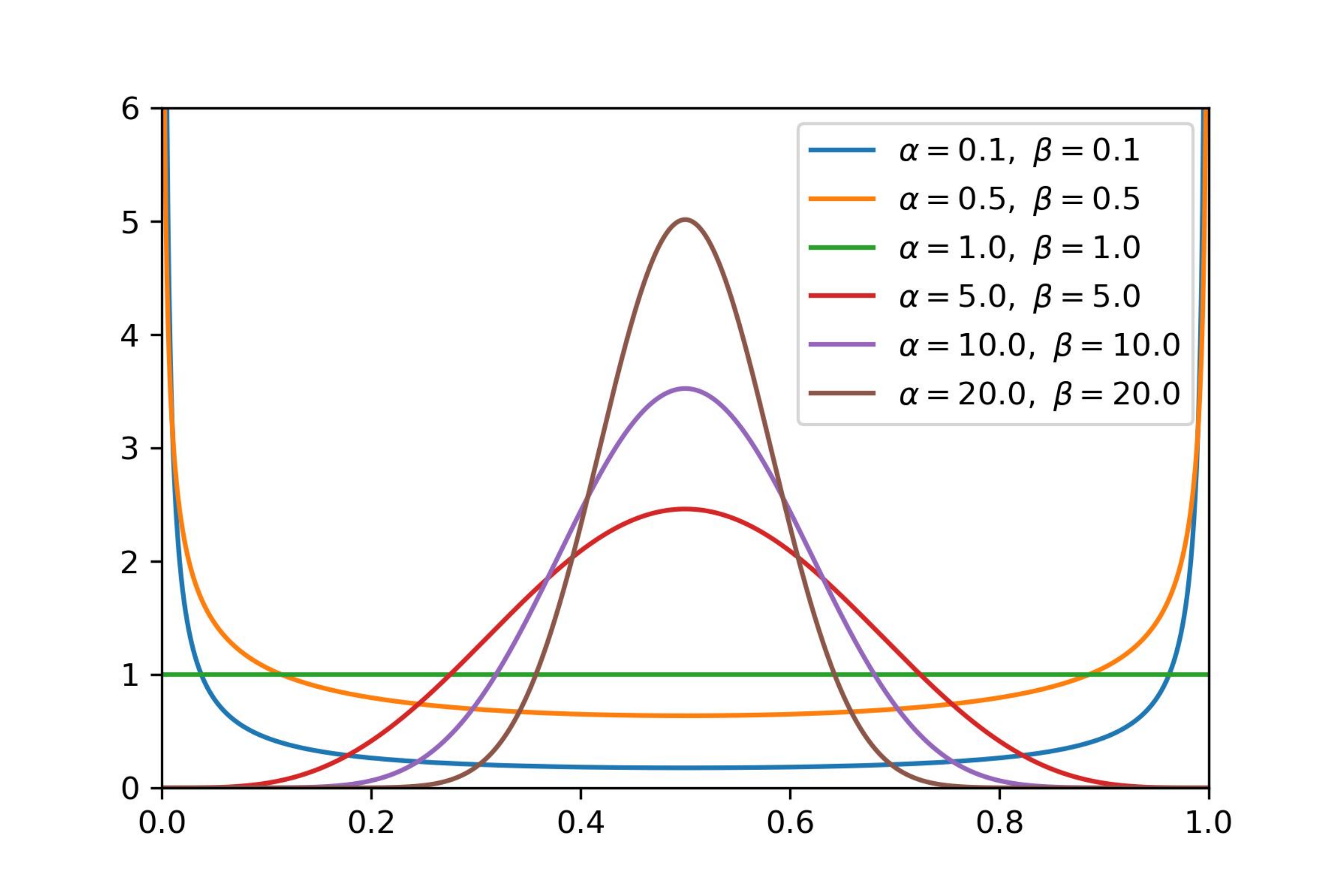}
		\vskip -0.17 in
		\caption{$\text{Beta}(\alpha, \alpha)$ pdf for varying $\alpha$.}
		\label{figure-A2}
	\end{minipage}%
	\begin{minipage}{.5\linewidth}
		\setlength\tabcolsep{8pt}
		\centering
		\renewcommand{\arraystretch}{1}
		\scalebox{0.85}{
			\begin{tabular}{cccccc}
				\toprule
				\textbf{$\alpha$} &\textbf{AURC} $\downarrow$ & \textbf{AUROC} $\uparrow$ & \textbf{FPR95} $\downarrow$  &\textbf{ACC} $\uparrow$ \\
				\midrule
				None  &9.52$\pm$0.49 &90.13$\pm$0.46 & 43.33$\pm$0.59 & 94.30$\pm$0.06\\
				\midrule
				0.1 & 9.30$\pm$2.01  & 92.04$\pm$0.81 & 46.69$\pm$5.72 & 92.76$\pm$0.91 \\
				0.5 & 7.56$\pm$1.51 & 91.87$\pm$1.43 & 44.20$\pm$3.76  & 94.00$\pm$0.25    \\
				1 & 6.58$\pm$0.61 & 92.19$\pm$0.36 & 39.17$\pm$1.77 & 94.59$\pm$0.19  \\
				5 & 6.08$\pm$0.88 & 92.46$\pm$1.02 & 37.44$\pm$0.79 & 94.89$\pm$0.15 \\
				10  &6.31$\pm$0.32 &92.09$\pm$0.36  &39.63$\pm$2.36 &94.98$\pm$0.20 \\
				20 & 5.96$\pm$0.90 & 92.45$\pm$0.22 & 35.70$\pm$0.72  & 95.12$\pm$0.17 \\
				\bottomrule
			\end{tabular}
		}
		\captionof{table}{Ablation study on $\alpha$.}
		\label{table-A3}
	\end{minipage}
\end{figure}

\noindent We conduct experiments to compare the effectiveness of different interpolation coefficient, which is drawn from the Beta distribution (refer Fig.~\ref{figure-A2}). Specifically, high values of $\alpha$ would encourage $\lambda \approx 0.5$. As shown in Table~\ref{table-A3}, large values of $\alpha$ (strong interpolations) lead to good performance. Since we aim to improve the exposure of low density regions, the interpolation should be strong to yield low confidence samples.

\subsection{Using Cutmix to transform outliers in OpenMix}
\noindent In our main manuscript, linear interpolation is applied to transform the outliers. An alternative way is to use non-linear strategy like CutMix \cite{yun2019cutmix}. From the results in Table~\ref{table-A4}, we observe that Cutmix based outlier transformation can yield comparable performance as mixup based.
\begin{table*}[h]
	\caption{Comparison between linear (Mixup) and non-linear (Cutmix) based outlier transformation.}
	\vskip -0.07in
	\label{table-A4}
	\setlength\tabcolsep{3pt}
	\centering
	\renewcommand{\arraystretch}{1.1}
	\scalebox{0.8}{
		\begin{tabular}{llccccccccc}
			\toprule
			\multirow{2}{*}{\textbf{Network}} &\multirow{2}{*}{\textbf{Method}} & \multicolumn{4}{c}{\textbf{CIFAR-10}} && \multicolumn{4}{c}{\textbf{CIFAR-100}}\\
			\cmidrule(lr){3-6} \cmidrule(lr){8-11}
			& &\textbf{AURC} $\downarrow$ & \textbf{AUROC} $\uparrow$ & \textbf{FPR95} $\downarrow$  &\textbf{ACC} $\uparrow$  &&\textbf{AURC} $\downarrow$ & \textbf{AUROC} $\uparrow$ & \textbf{FPR95} $\downarrow$  &\textbf{ACC} $\uparrow$ \\
			\midrule
			\multirow{3}{*}{ResNet110} 
			&MSP~{\scriptsize\textcolor{darkgray}{[ICLR17]}} \cite{hendrycks2017baseline}  &9.52$\pm$0.49 &90.13$\pm$0.46 & 43.33$\pm$0.59 & 94.30$\pm$0.06 &&89.05$\pm$1.39 &84.91$\pm$0.13 &65.65$\pm$1.72 &73.30$\pm$0.25\\
			&OpenMix (w/Mixup) &\bftab{6.31$\pm$0.32} &92.09$\pm$0.36  &39.63$\pm$2.36 &\bftab{94.98$\pm$0.20} &&\bftab{73.84$\pm$1.31} &85.83$\pm$0.22 &\bftab{64.22$\pm$1.35} &\bftab{75.77$\pm$0.35}\\
			&OpenMix (w/CutMix) &6.74$\pm$1.07 & \bftab{93.45$\pm$0.44} & \bftab{36.82$\pm$3.65} & 93.73$\pm$0.72 &  & 76.28$\pm$1.83 & \bftab{86.49$\pm$0.17} & 64.78$\pm$0.93 & 74.15$\pm$0.41 \\
			\midrule
			\multirow{3}{*}{WRNet} 
			&MSP~{\scriptsize\textcolor{darkgray}{[ICLR17]}} \cite{hendrycks2017baseline}  &4.76$\pm$0.62 &93.14$\pm$0.38 &30.15$\pm$1.98 &95.91$\pm$0.07 &&46.84$\pm$0.90  &88.50$\pm$0.44 & 56.64$\pm$1.33 & 80.76$\pm$0.18\\
			&OpenMix (w/Mixup) &\bftab{2.32$\pm$0.15} &\bftab{94.81$\pm$0.34}  &\bftab{22.08$\pm$1.86} &\bftab{97.16$\pm$0.10} &&\bftab{39.61$\pm$0.54} &\bftab{89.06$\pm$0.11}  &\bftab{55.00$\pm$1.29} &\bftab{82.63$\pm$0.06}\\
			&OpenMix (w/CutMix) &3.11$\pm$0.50 & 94.14$\pm$0.17 & 28.25$\pm$2.25 & 96.60$\pm$0.40 &  & 43.22$\pm$1.01 & 89.16$\pm$0.16 & 55.62$\pm$1.67 & 80.94$\pm$0.31 \\
			\midrule
			\multirow{3}{*}{DenseNet} 
			&MSP~{\scriptsize\textcolor{darkgray}{[ICLR17]}} \cite{hendrycks2017baseline} &5.66$\pm$0.45 &93.14$\pm$0.65 &38.64$\pm$4.70 &94.78$\pm$0.16 &&66.11$\pm$1.56 &86.20$\pm$0.04 &62.79$\pm$0.83 &76.96$\pm$0.20 \\
			&OpenMix (w/Mixup) &\bftab{4.68$\pm$0.72} &93.57$\pm$0.81  &\bftab{33.57$\pm$3.70} &\bftab{95.51$\pm$0.23} &&\bftab{53.83$\pm$0.93} &\bftab{87.45$\pm$0.18} &\bftab{62.22$\pm$1.15}  &\bftab{78.97$\pm$0.31}\\
			&OpenMix (w/CutMix) &5.44$\pm$0.50 & \bftab{93.80$\pm$0.13} & 37.28$\pm$2.15 & 94.48$\pm$0.39 &  & 68.80$\pm$6.96 & 86.46$\pm$0.57 & 63.99$\pm$2.41 & 75.92$\pm$1.24\\
			\bottomrule
		\end{tabular}
	}
	\vskip -0.1 in
\end{table*}

\begin{figure}[h]
	\begin{center}
		\centerline{\includegraphics[width=0.8\columnwidth]{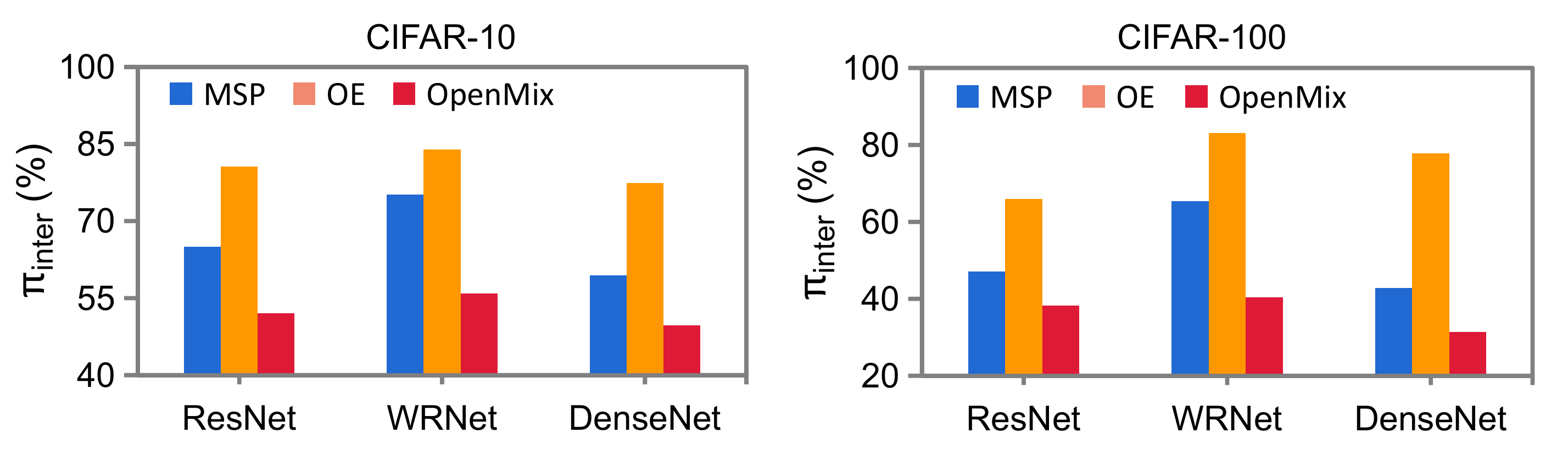}}
		\vskip -0.15 in
		\caption{Comparison of the inter-distance $\pi_{inter}$ of the deep feature space.}
		\label{figure-A3}
	\end{center}
	\vskip -0.25 in
\end{figure}

\begin{figure}[t]
	\begin{center}
		\centerline{\includegraphics[width=0.9\columnwidth]{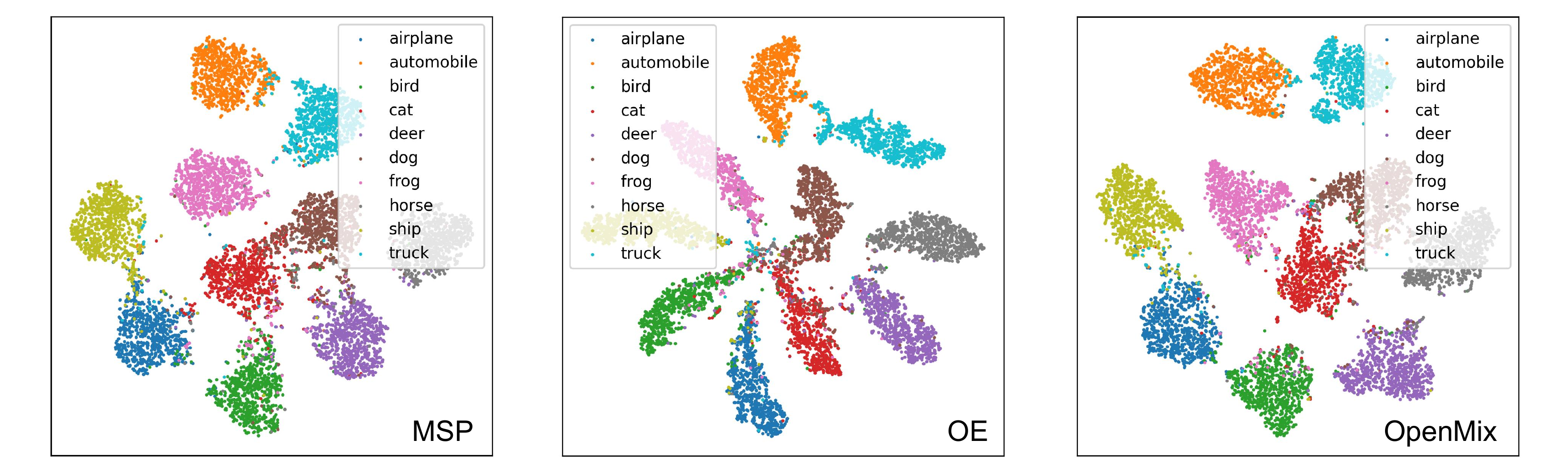}}
		\vskip -0.05 in
		\caption{Qualitative visualization of the deep feature space using TSNE \cite{vandermaaten08a}.}
		\label{figure-A4}
	\end{center}
	\vskip -0.3 in
\end{figure}

\subsection{More results of feature space distance and visualization}
\noindent In our main manuscript, we only plot the results of FSU due to the space limitation. Here, Fig.~\ref{figure-A3} plots the inter-class distance of the deep feature space. As can be observed, the inter-class distance with OE is observably enlarged, which indicates excessive feature compression and has negative influence for MisD. Our OpenMix leads to less compact feature distributions. Besides, Fig.~\ref{figure-A4} presents qualitative visualization to look at the effectiveness of OpenMix. Compared with MSP and OE, the feature distribution is smoother and the decision boundary is clearer, and the misclassified samples are mostly mapped to low-density regions in feature distribution.

\subsection{Ablation Study of each component in our method}
\noindent Table~\ref{table-A44} presents more results of each component in our method on WRNet and DenseNet.

\begin{table*}[h]
	\vskip 0.02in
	\caption{Ablation Study of each component in our method.}
	\vskip -0.07in
	\label{table-A44}
	\setlength\tabcolsep{5pt}
	\centering
	\renewcommand{\arraystretch}{1}
	\scalebox{0.83}{
		\begin{tabular}{llccccccccc}
			\toprule
			\multirow{2}{*}{\textbf{Network}}  &\multirow{2}{*}{\textbf{Method}} & \multicolumn{4}{c}{\textbf{CIFAR-10}} && \multicolumn{4}{c}{\textbf{CIFAR-100}}\\
			\cmidrule(lr){3-6} \cmidrule(lr){8-11}
			&&\textbf{AURC}  & \textbf{AUROC}  & \textbf{FPR95}   &\textbf{ACC}   &&\textbf{AURC}  & \textbf{AUROC}  & \textbf{FPR95}   &\textbf{ACC} \\
			\midrule
			\multirow{4}{*}{ResNet} 
			&MSP   &9.52 &90.13 & 43.33 & 94.30 &&89.05 &84.91 &65.65 &73.30\\
			&+ RC     & 9.55  & 91.15 & 40.03 & 94.02 &  & 94.31 & 85.53 & 65.78 & 71.44 \\
			&+ OT     & 12.38 & 87.13 & 61.83 & 93.84 &  & 99.86 & 82.51 & 72.94 & 72.62 \\
			&OpenMix &\bftab{6.31} &\bftab{92.09}  &\bftab{39.63} &\bftab{94.98} &&\bftab{73.84} &\bftab{85.83} &\bftab{64.22} &\bftab{75.77}\\
			\midrule
			\multirow{4}{*}{WRNet} 
			&MSP   &4.76 &93.14 &30.15 &95.91 &&46.84  &88.50 & 56.64 & 80.76\\
			&+ RC     & 4.28  & 93.95 & 30.05 & 95.62 &  & 54.32 & 88.08 & 60.17 & 78.69 \\
			&+ OT     & 5.75  & 90.71 & 49.69 & 95.77 &  & 54.38 & 86.24 & 64.68 & 80.12 \\
			&OpenMix &\bftab{2.32} &\bftab{94.81}  &\bftab{22.08} &\bftab{97.16} &&\bftab{39.61} &\bftab{89.06}  &\bftab{55.00} &\bftab{82.63}\\
			\midrule
			\multirow{4}{*}{DenseNet} 
			&MSP    &5.66 &93.14 &38.64 &94.78 &&66.11 &86.20 &62.79 &76.96 \\
			&+ RC     & 6.04  & 93.07 & 37.55 & 94.56 &  & 70.73 & 86.78 & 64.36 & 75.21 \\
			&+ OT     & 10.46 & 87.76 & 62.85 & 94.29 &  & 76.92 & 84.09 & 70.55 & 75.78 \\
			&OpenMix &\bftab{4.68} &\bftab{93.57}  &\bftab{33.57} &\bftab{95.51} &&\bftab{53.83} &\bftab{87.45} &\bftab{62.22}  &\bftab{78.97}\\
			\bottomrule	
		\end{tabular}
	}
\end{table*}

\section{Additional Analysis}
\subsection{More insights: Impact of feature space uniformity for OOD detection and MisD}
\noindent We provide more insights about the connection between feature space uniformity (FSU, refer to Sec. 3.2 for detailed definition) and OOD detection, MisD performance. According to the familiarity hypothesis \cite{dietterich2022familiarity}, the features are less activated for OOD samples from unknown classes than that for ID samples. Therefore, MisD is more difficult than OOD detection, and the FSU has different impact on those two tasks. In what follows, we provide more illustration based on Fig.~\ref{figure-A5}. Specifically,
\begin{figure}[h]
	\begin{center}
		\centerline{\includegraphics[width=\columnwidth]{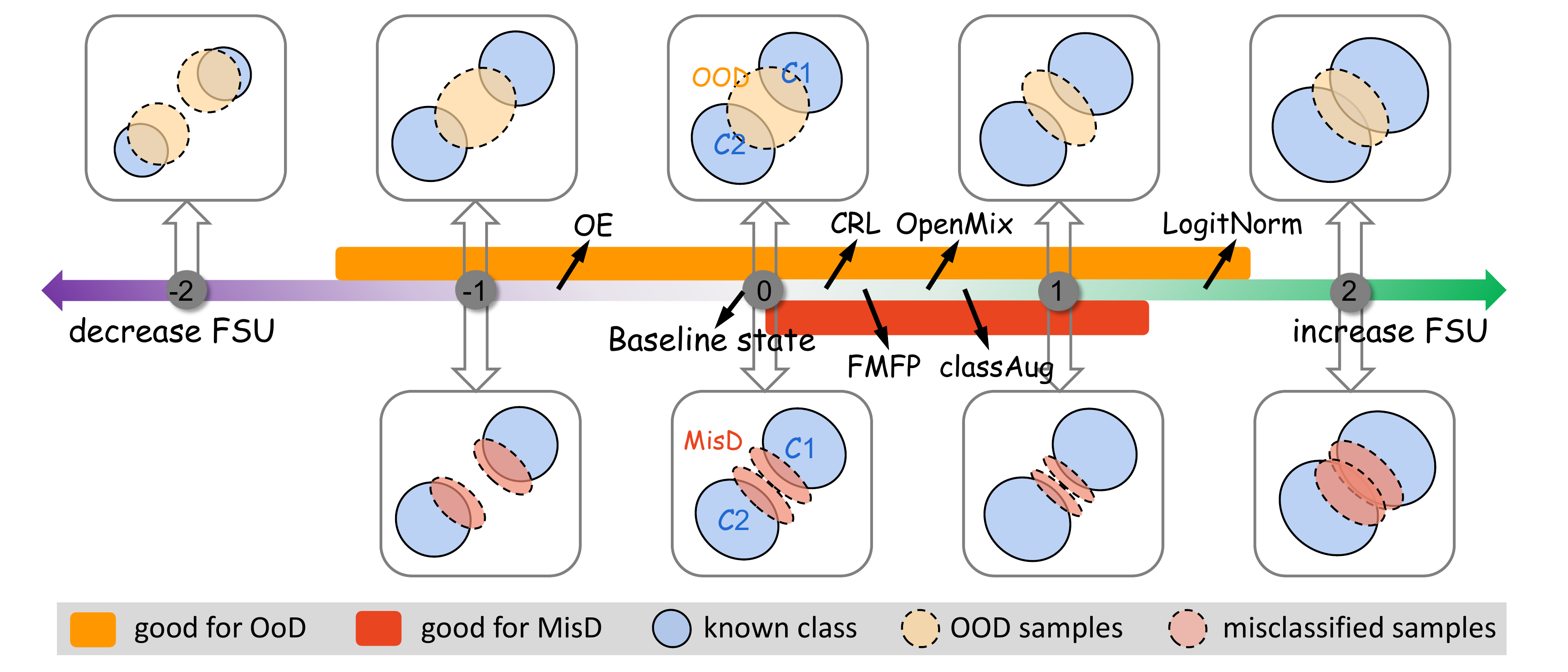}}
		\vskip -0.05 in
		\caption{Illustration of how the change of FSU affects the OOD detection and MisD performance.}
		\label{figure-A5}
	\end{center}
	\vskip -0.2 in
\end{figure}
\begin{itemize}
	\item For OOD detection, at the baseline state ({\ttfamily state 0}), the OOD distribution has some overlap with ID distribution. \ding{172}  When decreasing the FSU, the distribution of known classes is compressed and the overlap between OOD and ID samples could be reduced ({\ttfamily state -1}). However, when further decreasing the FSU, the ID distribution could be much over-compact and the model maps most of the OOD samples to the ID region ({\ttfamily state -2}), leading to worse OOD detection performance. \ding{173} When increasing the FSU, more OOD samples could be mapped to low density regions ({\ttfamily state 1}). However, when further increasing the FSU, the ID distribution would be under-activated ({\ttfamily state 2}), leading to worse separation between ID and OOD distribution.
	\item For MisD, compared with OOD samples, the misclassified samples are ID and closer to the correct samples. Therefore, MisD performance is more sensitive to the change of FSU. As a result, \ding{172} decreasing the FSU would easily lead to more overlap between correct and misclassified ID samples ({\ttfamily state -1}). To improve the separation, it more helpful to \ding{173} increase the FSU ({\ttfamily state 1}), making the features be less activated for misclassified samples. However, when further increasing the FSU, the distribution of correct samples would be under-activated ({\ttfamily state 2}), leading to worse separation between correct and wrong data.
\end{itemize}
\noindent In conclusion, both over-compact and over-dispersive feature distributions are harmful for OOD detection and MisD. To effectively detect OOD and misclassified samples, it is better to increase the FSU to a proper level. In Fig.~\ref{figure-A5}, the \textcolor{orange}{orange} region is good for OOD detection, while the \textcolor{red}{red} region is good for MisD. The common region between the orange and red is desirable for detecting OOD and misclassified samples in a unified manner. In addition, we compute the FSU of several representative methods (OE \cite{hendrycks2019deep}, MSP \cite{hendrycks2017baseline}, CRL \cite{MoonKSH20}, FMFP \cite{zhu2022rethinking}, OpenMix, classAug \cite{zhu2022learning}, LogitNorm \cite{wei2022logitnorm}) and mark the corresponding position in Fig.~\ref{figure-A5}. The effect of them is consistent with our analysis.

\subsection{Theoretical analysis: OpenMix increases the exposure of low density regions}
\noindent In standard training, with cross-entropy loss and one-hot label, there are few uncertain samples are mapped to low density regions.
An intuitive interpretation of the effectiveness of OpenMix is it increases the exposure of low density regions in feature space by synthesizing and learning the mixed samples. We provide a theoretical justification showing that our method can increase the sample density in the original low-density regions.

\setParDis
\noindent 
Suppose we have a known class consisting of samples drawn from probability density function $f(x)$, and an outlier distribution $f_{\text{ood}}(x)$ that is farther away from $f(x)$. By applying linear interpolation (\emph{i.e.}, Mixup) between ID distribution $f(x)$ and outlier distribution $f_{\text{ood}}(x)$, we can get a mixed set. Denote $f_{\text{mix}}(x)$ the bimodal distribution that represents the probability density function of mixed samples. With integration of $f(x)$ and $f_{\text{mix}}(x)$, the new data probability density function is denoted as $\bar f(x)=\frac{1}{2}\left(f(x)+f_{\text{mix}}(x)\right)$. The following theorem shows that the probability density on the subset $S=\{x||x^{\tau}v|>C, x\in\mathcal{R}^d\}$ is enlarged, in which $C$ is a sufficiently large constant and $v\in\mathcal{R}^d $ is the certain direction. For example, for the single dimensional case with variance $\sigma^2$, the density is guaranteed to be enlarged in the set $S'=\{x||x|>1.5\sigma,\mu=1\}$, which is exactly the low-density area for the Guassian distribution.

\begin{theorem}
	Let $f(x)$ and $f_{\text{mix}}(x)$ be the probability density functions defined as follows,
	$$
	f(x)=\frac{1}{(2\pi)^{d/2}|\Sigma|^{1/2}}\exp\left(-\frac{x^{\tau}\Sigma^{-1}x}{2}\right),
	$$
	and
	$$
	f_{\text{mix}}(x)=\frac{1}{2(2\pi)^{d/2}|\Sigma|^{1/2}}\exp\left(-\frac{(x-\mu)^{\tau}\Sigma^{-1}(x-\mu)}{2}\right)+\frac{1}{2(2\pi)^{d/2}|\Sigma|^{1/2}}\exp\left(-\frac{(x-\mu)^{\tau}\Sigma^{-1}(x-\mu)}{2}\right),
	$$
	where $x=(x_1,\cdots,x_d)\in\mathbb{R}^d$, $\mu\in\mathbb{R}^d$ and $\Sigma$ are the correspondingly mean vector and positive-definite covariance matrix.
	Assume that $\mu=\Sigma^{1/2}\bar\mu$, where $\|\bar u\|=1$ can be chosen arbitrary, then , it follows that 
	$$f(x)<\bar f(x), ~\text{for any} ~x\in S'=\{x||x^{\tau}v|>1.5, v= \Sigma^{-1/2}\bar\mu\}.$$
\end{theorem}	
\begin{proof}
	Denote $g(x)=\bar f(x)-f(x)$, we have
	$$
	\begin{aligned}
	g(x)=&\frac{1}{2}\left(f(x)+f_{\text{mix}}(x)\right)-f(x)=\frac{1}{2}\left(f_{\text{mix}}(x)-f(x)\right)\\
	=&\frac{1}{4(2\pi)^{d/2}|\Sigma|^{1/2}}\exp\left(-\frac{x^{\tau}\Sigma^{-1}x}{2}\right)\left[\exp\left(-\frac{\mu^{\tau}\Sigma^{-1}\mu}{2}\right)\left(\exp\left(x\Sigma^{-1}\mu\right)+\exp\left(-x\Sigma^{-1}\mu\right)\right)-2\right]	\\
	\end{aligned}    
	$$
	In what follows, we show that $g(x)>0$ on the region $x\in S'$.
	Firstly, it is trivial to see that $g(0)=2\exp\left(-\frac{\mu^{\tau}\Sigma^{-1}\mu}{2}\right)-2<0$.
	To analyze the property of $g(x)$, we need to analyze the following function:
	$$\begin{aligned}
	h(x)=&\exp\left(x^{\tau}\Sigma^{-1}\mu\right)+\exp\left(-x^{\tau}\Sigma^{-1}\mu\right)-2\exp\left(\frac{\mu^{\tau}\Sigma^{-1}\mu}{2}\right)\\
	=&\exp\left(x^{\tau}\Sigma^{-1/2}\bar\mu\right)+\exp\left(-x^{\tau}\Sigma^{-1/2}\bar\mu\right)-2\exp\left(\frac{\bar\mu^{\tau}\bar\mu}{2}\right)\\
	=&\exp\left(x^{\tau}\Sigma^{-1/2}\bar\mu\right)+\exp\left(-x^{\tau}\Sigma^{-1/2}\bar\mu\right)-2\exp\left(\frac{1}{2}\right)\\
	\end{aligned}$$
	Noticing that $\exp(x)+\exp(-x)$ is an even function and it is increasing with respect to $|x|$, thus for $h(x)$, there exists a positive constant $m$ such that, $h(x)>0$ when $|x^{\tau}\Sigma^{-1/2}\bar\mu|\geq m$. We can see that $\exp\left(1.5\right)+\exp\left(-1.5\right)-2\exp\left(\frac{1}{2}\right)>0$, which means $m\geq1.5$. That means $g(x)>0$ when $|x^{\tau}\Sigma^{-1/2}\bar\mu|\geq1.5 $, thus we complete the proof.
\end{proof}

{\small
	\bibliographystyle{ieee_fullname}
	\bibliography{reference}
}
\end{document}